\documentclass[10pt]{article} % For LaTeX2e
\usepackage[accepted]{rlc}
\usepackage{amssymb}            % Defines common symbols like \mathbb R
\usepackage{mathtools}          % Extends amsmath, providing common math tools
\usepackage{mathrsfs}           % Enables \mathscr, which can work in cases that \mathcal does not
\mathtoolsset{showonlyrefs}     % Only number equations that are referenced (optional)
\usepackage{graphicx}           % For including images
\usepackage{subcaption}         % Allows for the use of subfigures and subcaptions
\usepackage[space]{grffile}     % For spaces in image names
\usepackage{url}                % For displaying urls
\usepackage{amsmath,amssymb,amsthm}
\usepackage{mathtools}
\usepackage{algorithm}
\usepackage[noend]{algpseudocode}
\usepackage{enumitem}

\usepackage{booktabs}
\usepackage[normalem]{ulem}

\newtheorem{theorem}{Theorem}
\newtheorem*{theorem-no}{Theorem}
\newtheorem{lemma}{Lemma}
\newtheorem*{lemma-no}{Lemma}

\newtheorem{remark}{Remark}
\newtheorem{corollary}{Corollary}

\newtheorem{definition}{Definition}
\newtheorem{assumption}{Assumption}

\newcommand{\expect}{\mathbb{E}}
\newcommand{\identity}{\mathrm{I}}
\newcommand\norm[1]{\left\lVert#1\right\rVert}

\title{A Tighter Convergence Proof of Reverse Experience Replay}

\author{Nan Jiang,  Jinzhao Li, Yexiang Xue  \\
    \texttt{\{jiang631, li4255, yexiang\}@purdue.edu} \\
    Department of Computer Science\\
    Purdue University, USA}

\begin{document}

\maketitle
\begin{abstract}

In reinforcement learning, Reverse Experience Replay (RER) is a recently proposed algorithm that attains better sample complexity than the classic experience replay method. RER requires the learning algorithm to update the parameters through consecutive state-action-reward tuples in reverse order. However, the most recent theoretical analysis only holds for a minimal learning rate and short consecutive steps, which converge slower than those large learning rate algorithms without RER. In view of this theoretical and empirical gap, we provide a tighter analysis that mitigates the limitation on the learning rate and the length of consecutive steps. Furthermore, we show theoretically that RER converges with a larger learning rate and a longer sequence.
\end{abstract}

%%%%%%%%%%%%%%%%%%%%%%%%%%%%%%%%%%%%%%%%%%%%%%%%%%%%%%%%%%%%%%%%
%% Section: Submission of papers to RLC
%%%%%%%%%%%%%%%%%%%%%%%%%%%%%%%%%%%%%%%%%%%%%%%%%%%%%%%%%%%%%%%%

\section{Introduction}

Reinforcement Learning (RL) is highly successful for a variety of practical
problems in the realm of long-term decision-making. Experience Replay (ER) of historical trajectories plays a vital role in RL algorithms~\citep{DBLP:journals/ml/Lin92,nature2015dqn}. The trajectory is a sequence of transitions, where each transition is a state, action, and reward tuple. The memory space used to store these experienced trajectories is noted as the replay buffer. The methods to sample transitions from the replay buffer determine the rate and stability of the convergence of the learning algorithms. 

Recently, Reversed Experience Replay (RER)~\citep{DBLP:conf/corl/FlorensaHWZA17,DBLP:journals/corr/Rotinov2019,DBLP:conf/nips/LeeCC19,DBLP:conf/iclr/Agarwal2022}  is an approach inspired by the hippocampal reverse replay mechanism in human and animal neuron~\citep{foster2006reverse,AMBROSE20161124,doi:10.1073/pnas.2011266118}. Theoretical analysis shows that RER improves the convergence rate towards optimal policies in comparison with ER-based algorithms. 
Unlike ER, which samples transitions uniformly~\citep{DBLP:conf/aaai/HasseltGS16} (known as classic experience replay) or weightily ~\citep{DBLP:journals/corr/SchaulQAS15} (known as prioritized experience replay) from the replay buffer, 
RER samples consecutive sequences of transitions from the buffer and reversely fed into the learning algorithm. 

However, the most recent theoretical analysis on RER with $Q$-learning only holds for a minimal learning rate and short consecutive steps~\citep{DBLP:conf/iclr/Agarwal2022}, which converges slower than classic $Q$-learning algorithm (together with ER) with a large learning rate.  We attempt to bridge the gap between theory and practice for the newly proposed reverse experience replay algorithm.

In this paper, we provide a tighter analysis that relaxes the limitation on the learning rate and the length of the consecutive transitions. Our key idea is to transform the original problem involving a giant summation (shown in Equation~\ref{eq:gamma-gamma}) into a combinatorial counting problem (shown in Lemma~\ref{lem:combi-cases}), which greatly simplifies the whole problem. We hope the new idea of transforming the original problem into a combinatorial counting problem can enlighten other relevant domains.
Furthermore, we show in Theorem~\ref{thm:main-cvg} that RER converges faster with a larger learning rate $\eta$ and a longer consecutive sequence $L$ of state-action-reward tuples.

\section{Preliminaries} \label{sec:prelim}

\paragraph{Markov Decision Process}  
We consider a Markov decision process (MDP) with discounted rewards, noted as $\mathcal{M}=(\mathcal{S}, \mathcal{A}, {P}, r, \gamma)$. 
Here $\mathcal{S}\subset\mathbb{R}^d$ is the set of states, $\mathcal{A}$ is the set of actions, and $\gamma\in(0,1)$ indicates the discounting factor. We use $P: \mathcal{S} \times \mathcal{A} \times  \mathcal{S}\to [0,1]$ as the transition probability kernel of MDP. 
For each pair $(s,a)\in\mathcal{S}\times \mathcal{A}$, $P( s'|s, a) $ is the probability of transiting to state $s'$ from state $s$ when action $a$ is executed. 
The reward function is $r:\mathcal{S} \times \mathcal{A} \to[-1,1] $, such that $r(s,a)$ is the immediate reward from state $s$ when action $a$ is executed~\citep{DBLP:books/wi/Puterman94}. 
The policy $\pi$ is a mapping from states to a distribution over the set of actions: $\pi(s):\mathcal{A}\to [0,1]$, for $s\in \mathcal{S}$. 
A trajectory is noted as $\{(s_t,a_t,r_t)\}_{t=0}^{\infty}$, where $s_t$  (respectively $a_t$) is the state (respectively the action taken) at time $t$, $r_t=r(s_t,a_t)$ is the reward received at time $t$, and $(s_t,a_t,r_t,s_{t+1})$ is the $t$-step transition.

\paragraph{Value Function and $Q$-Function} The value function of a policy $\pi$ is noted as $V^\pi:\mathcal{S}\to\mathbb{R}$. For  $s\in\mathcal{S}$, $ V^{\pi}(s):=\mathbb{E}\left[\sum_{t=0}^\infty\gamma^tr(s_t,a_t|s_0=s)\right]$,
which is the expected discounted cumulative reward received when 1) the initial state is $s_0=s$, 2) the actions are taken based on the policy $\pi$, \textit{i.e.}, $a_t\sim\pi(s_t)$,  for $t\ge 0$.  3) the trajectory is generated by the transition kernel, \textit{i.e.}, $s_{t+1}\sim P(\cdot|s_t,a_t)$, for all $t\ge 0$.
Similarly, let $Q^{\pi}:\mathcal{S}\times \mathcal{A}\to\mathbb{R}$ be the action-value function (also known as the $Q$-function) of a policy $\pi$. For  $(s,a)\in\mathcal{S}\times \mathcal{A}$, it is defined as $ Q^{\pi}(s,a):=\mathbb{E}\left[\sum_{t=0}^\infty\gamma^tr(s_t,a_t|s_0=s,a_0=a)\right].$

There exists an optimal policy, denoted as $\pi^*$ that maximizes $Q^{\pi}(s,a)$ uniformly over all state-action pairs  $(s,a)\in\mathcal{S}\times \mathcal{A}$~\citep{watkins1989learning}. We denote $Q^*$ as the $Q$-function corresponding to $\pi^*$, \textit{i.e.}, $Q^*=Q^{\pi^*}$. 
The Bellman operator $\mathcal{T}$ on a $Q$-function is defined as: for $(s,a)\in\mathcal{S}\times \mathcal{A}$, 
\begin{align*}
\mathcal{T}(Q)(s,a):=r(s,a)+\gamma \mathbb{E}_{s'\sim P(\cdot|s,a)}\left[\max_{a'\in\mathcal{A}}Q(s',a')\right].
\end{align*}
The optimal $Q$-function $Q^*$ is the unique fixed point of the Bellman operator~\citep{DBLP:journals/mor/BertsekasY12}. % make sure this is consistent across texts.

\paragraph{$Q$-learning}~~ The $Q$-learning algorithm is a model-free algorithm to learn $Q^*$~\citep{DBLP:journals/ml/WatkinsD92}.  
The high-level idea is to find the fixed point of the Bellman operator. 
Given the trajectory $\{(s_t,a_t,r_t)\}_{t=0}^\infty$ generated by some underlying behavior policy $\pi'$, the asynchronous $Q$-learning algorithm estimates a new $Q$-function $Q_{t+1}:\mathcal{S}\times \mathcal{A}\to \mathbb{R}$ at each time. At time $t\ge 0$, given a transition $(s_t,a_t,r_t,s_{t+1})$, the algorithm update as follow:
\begin{equation}\label{eq:async_q}
\begin{aligned}
Q_{t+1}(s_{t},a_{t})&=(1-\eta)Q_{t}(s_{t},a_{t})+\eta \mathcal{T}_{t+1}(Q_t)(s_{t+1},a_{t}),\\
Q_{t+1}(s,a)&=Q_{t}(s,a),&\text{ for all } (s,a)\neq (s_t,a_t).  
\end{aligned}
\end{equation}
Here $\eta\in(0,1)$ is the learning rate and $\mathcal{T}_{t+1}$ is the \textit{empirical} Bellman operator: $\mathcal{T}_{t+1}(Q_{t})(s_{t},a_t):=r(s_t,a_t)+\gamma \max_{a'\in\mathcal{A}}Q_t(s_{t+1},a')$. 
Under mild conditions, $Q_t$ will converge to the fixed point of the Bellman operator and hence to $Q^*$. When the state space $\mathcal{S}$ is small, a tabular structure cab be used to store the values of $Q_t(s,a)$ for $(s,a)\in \mathcal{S}\times \mathcal{A}$.

\paragraph{$Q$-learning with Function Approximation}  When the state space $\mathcal{S}$ is large, the asynchronous $Q$-learning in Equation~\eqref{eq:async_q} cannot be applied since it needs to loop over a table of all states and actions. 
In this case, function approximation is 
brought into $Q$-learning. Let $Q^w:\mathcal{S}\times \mathcal{A}\to \mathbb{R}$ be an approximated $Q$-function, which is typically represented with a deep neural network~\citep{nature2015dqn} and $w$ denotes the parameters of the neural network. $Q^w$ is often called the $Q$-network.
%
%The task now changes to estimating the parameters $w$. 
Given a batch of transitions $\{(s_{t_i},a_{t_i},r_{t_i},s_{t_i+1})\}_{i=1}^m$,  we define $y_{t_i}$ as the image of $Q^{w'}(s_{t_i}, a_{t_i})$ under the empirical Bellman operator, that is:
\begin{equation*}
y_{t_i}:=r_{t_i}+\gamma\max_{a'\in\mathcal{A}} Q^{w'}(s_{{t_i}+1},a'), \quad \text{ for } 1\le i\le m
\end{equation*}
where $w'$ represents the parameters in \textit{target} neural network. Parameters $w'$ are synchronized to $w$ every $T_{target}$ steps of Stochastic Gradient Descent (SGD). Since $Q^*$ is the fixed point of the Bellman operator, $y_{t_i}$ should match $Q^w(s_{t_i}, a_{t_i})$ when $Q^w$ converges to $Q^*$. Hence, learning is done via minimizing the following objective using SGD: $\ell(w)=\frac{1}{m}\sum_{i=1}^m\|y_{t_i} -Q^w(s_{t_i}, a_{t_i})\|_2^2$.

\paragraph{Experience Replay} 
For the $Q$-learning with function approximation, the new trajectories are generated by executing a behavioral policy, which are then saved into the \textit{replay buffer}, noted as $\mathcal{B}$.
When learning to minimize $\ell(w)$, 
SGD is performed on batches of \textit{randomly sampled} transitions from the replay buffer. This process is often called Experience Replay (ER)~\citep{DBLP:journals/ml/Lin92,DBLP:journals/tit/LiWCGC22}.
To improve the stability and convergence rate of $Q$-learning, follow-up works sample transitions from the replay buffer with non-uniform probability distributions.
Prioritized experience replay favors those transitions with a large temporal difference errors~\citep{DBLP:journals/corr/SchaulQAS15,DBLP:journals/jair/SaglamMCK23}. 
Discor~\citep{DBLP:conf/nips/Kumar0L20}  favors those transitions with small Bellman errors. LaBER proposes a generalized  TD error to reduce the variance of gradient and improve learning stability~\citep{DBLP:conf/icml/LahireGR22}.  Hindsight experience replay uses imagined outcomes by relabeling goals in each episode, allowing the agent to learn from unsuccessful attempts as if they were successful~\citep{DBLP:conf/nips/AndrychowiczCRS17}.

\paragraph{Reverse Experience Replay} is a recently proposed variant of experience replay~\citep{DBLP:conf/iclr/GoyalBFSLLLB19,DBLP:conf/icml/BaiWHHG0W21,DBLP:conf/iclr/Agarwal2022}. RER samples \textit{consecutive} sequences of transitions from the replay buffer. 
The $Q$-learning algorithm updates its parameters by performing in the \textit{reverse} order of the sampled sequences. 
Compared with ER, RER converges faster towards the optimal policy  empirically~\citep{DBLP:conf/nips/LeeCC19} and theoretically~\citep{DBLP:conf/iclr/Agarwal2022}, under tabular and linear MDP settings. 
One intuitive explanation of why RER works 
is to consider a sequence of consecutive transitions $s_1\xrightarrow{a_1, r_1} s_2 \xrightarrow{a_2, r_2} s_3$. 
Incorrect $Q$-function estimation of $Q(s_2, a_2)$ will affect the estimation of $Q(s_1, a_1)$. Hence, 
reverse order updates allow the $Q$-value updates of $Q(s_1, a_1)$ to use the most up-to-date value of $Q(s_2, a_2)$, hence accelerating the convergence.

\subsection{Problem Setups for Reverse Experience Replay}
\paragraph{Linear MDP Assumption} In this paper, we follow the definition of linear MDP from~\cite{DBLP:conf/icml/ZanetteLKB20}, which states that the reward function can be written as the inner product of the parameter $w$ and the feature function $\phi$. Therefore, the $Q$ function depends only on $w$ and the feature vector $\phi(s, a)\in \mathbb{R}^d$ for state $s\in\mathcal{S}$ and action $a\in\mathcal{A}$.

\begin{assumption}[Linear MDP setting from~\citealp{DBLP:conf/icml/ZanetteLKB20}]~\label{def:Linear} 
There exists a vector $w \in \mathbb{R}^d$ such that $R(s, a;w) = \langle w, \phi(s, a)\rangle$, and the transition probability is proportional to its corresponding feature $\mathcal{P}(\cdot|s, a) \propto \phi(s, a)$. 
Therefore, the optimal Q-function is $Q^*(s, a;w^*) = \langle w^*, \phi(s, a)\rangle$ for every $s\in \mathcal{S},a\in \mathcal{A}$.
\end{assumption}

The assumption~\ref{def:Linear} is the current popular Linear MDP assumption that allows us to quantify the convergence rate (or sample complexity) for the $Q$-learning algorithm~\citep{DBLP:conf/icml/ZanetteLKB20,DBLP:conf/iclr/Agarwal2022}. We need the following additional assumptions to get the final convergence rate result.
Assume the sequence of consecutive transitions is of length $L$ and the constant learning rate in the gradient descent algorithm is $\eta$.

\begin{assumption}[from \cite{DBLP:conf/icml/ZanetteLKB20}] \label{asump:phi-kappa}
The MDP has zero inherent Bellman error and $\phi(s,a)^\top \phi(s,a)\le 1$ for all $(s,a)\in\mathcal{S}\times \mathcal{A}$. There exists constant $\kappa>0$, such that $\expect_{(s,a)\sim\mu} \phi(s,a)\phi(s,a)^\top\succeq {\identity}/{\kappa}$. Here $\mu$ is the stationary distribution over all the state-action pairs of the Markov chain determined by the transition kernel and the policy.
\end{assumption}
\begin{remark}\label{rem:linear-expect}
Suppose we pick a set of state-action tuples $\mathcal{L}=\{(s,a)|(s,a)\in\mathcal{S}\times \mathcal{A}\}$, which may contains duplicated tuples. By linearity of expectation, we have: $\mathbb{E}_{\mu}\left(\sum_{(s,a)\in \mathcal{L}} \phi(s,a)\phi(s,a)^\top\right)=\sum_{ \mathcal{L}}\mathbb{E}_{(s,a)\sim\mu}\left( \phi(s,a)\phi(s,a)^\top\right) \succeq \frac{|\mathcal{L}|}{\kappa} \identity$. Here $|\mathcal{L}|$ indicates the number of state-action tuples in this set.
\end{remark}

\begin{definition} \label{def:phi}
Given the feature function $\phi: \mathcal{S}\times \mathcal{A}\to \mathbb{R}^d$.
Denote the largest inner product between parameter $w$ and the feature function $\phi$ as $\|w\|_{\phi}=\sup_{(s,a)}|\langle\phi(s,a),w\rangle|$. 
\end{definition}
\begin{definition} \label{def:gamma}
Let $\identity$ be an identity matrix of dimension $d\times d$ and $\eta\in\mathbb{R}$ as the learning rate.
Define matrix $\Gamma_{l}$ recursively as follow:
\begin{align}
    \Gamma_{l}\coloneqq\begin{cases}
    \identity &\text{ for } l=0, \\
    \left(\identity-\eta \phi_{L+1-l}\phi^\top_{L+1-l}\right) \Gamma_{l-1} & \text{ for } 1\le l\le L,
    \end{cases}
\end{align}
where we use the simplified notation $\phi_{L+1-l}$ to denote $\phi(s_{L+1-l},a_{L+1-l})$.
The explicit form for $\Gamma_{L}$ is: 
\begin{equation} 
\Gamma_{L}=\left(\identity-\eta \phi_{1}\phi^\top_{1}\right)\left(\identity-\eta \phi_{2}\phi^\top_{2}\right)\ldots \left(\identity-\eta \phi_{L}\phi^\top_{L}\right)=\prod_{l=1}^L\left(\identity-\eta \phi_{l}\phi^\top_{l}\right)
\end{equation}
\end{definition}

The semantic interpretation of $\Gamma_L$ in Definition~\ref{def:gamma} is that it represents the coefficient of the bias term in the error analysis of the learning algorithm's parameter (as outlined in Lemma~\ref{lemma:error-decomp}). This joint product arises because the RER algorithm updates the parameter over a subsequence of consecutive transitions of length $L$. The norm of $\Gamma_L$ is influenced by both the sequence length $L$ and the learning rate $\eta$. When the norm of $\Gamma_L$ is small, the parameters of the learning model converge more rapidly to their optimal values.

\section{Methodology} \label{sec:method}

\subsection{Motivation}

Let $\mu$ denote the stationary distribution of the state-action pairs in the MDP, $\eta$ be the learning rate of the gradient descent algorithm, and $L$ the length of consecutive transitions processed by the RER algorithm. Previous work~\citep[Lemma 8 and Lemma 14]{DBLP:conf/iclr/Agarwal2022} established that when $\eta L \le \frac{1}{3}$, the following inequality holds:
\begin{equation} \label{eq:previous-bound}
\mathbb{E}_{(s,a)\sim\mu} \left[\Gamma_{L}^\top \Gamma_{L}\right] \preceq\identity-\eta\sum_{l=1}^L\mathbb{E}_{(s,a)\sim\mu}\left[\phi_{l}\phi^\top_{l}\right]\preceq \left(1-\frac{\eta L}{\kappa}\right)\identity,
\end{equation}
where the matrix $\Gamma_L\in\mathbb{R}^{d\times d}$ is defined in Definition~\ref{def:gamma} and serves as a ``coefficient'' in the convergence analysis, as outlined in Lemma~\ref{lemma:error-decomp}. The positive semi-definite relation $\preceq$ between two matrices is defined in Definition~\ref{def:psd}. Here, $\identity$ represents an identity matrix of dimension $d \times d$, and the coefficient $\kappa > 0$ is introduced in Assumption~\ref{asump:phi-kappa}.  The matrix $\Gamma_L$ was mentioned in \citep[Appendix E, Equation 5]{DBLP:conf/iclr/Agarwal2022}, but we provide a formal definition here and streamline the original expression by removing unnecessary variables.

% =======
The condition in Equation~\eqref{eq:previous-bound} was further incorporated into the convergence requirement in~\citep[Theorem 1]{DBLP:conf/iclr/Agarwal2022}. It suggests that the RER algorithm cannot handle sequences of consecutive transitions that are too long (corresponding to a large $L$) or use a learning rate that is too large (i.e., $\eta$). This presents a major limitation between the theoretical justification and real-world application of the RER algorithm. In this work, we address this gap by providing a tighter theoretical analysis that relaxes the constraint $\eta L \le 1/3$.

We begin by explaining the main difficulty in upper-bounding the term $\mathbb{E}_{(s,a)\sim\mu} \left[\Gamma_{L}^\top \Gamma_{L}\right]$. According to Definition~\ref{def:gamma}, we can expand $\Gamma_{L}^\top$ as $\Gamma_{L}^\top = \left(\identity - \eta \phi_{L} \phi^\top_{L}\right)\cdots\left(\identity - \eta \phi_{1} \phi^\top_{1}\right)$. Using the linearity of expectation, we expand the entire joint product $\Gamma_{L}^\top \Gamma_{L}$ under the expectation as follows:
\begin{align}\label{eq:gamma-gamma}
\mathbb{E}_{(s,a)\sim\mu} \left[\Gamma_{L}^\top \Gamma_{L}\right] &= \mathbb{E}_{(s,a)\sim\mu} \left[\left(\identity - \eta \phi_{L} \phi^\top_{L}\right)\cdots\left(\identity - \eta \phi_{1} \phi^\top_{1}\right)\left(\identity - \eta \phi_{1} \phi^\top_{1}\right)\cdots\left(\identity - \eta \phi_{L} \phi^\top_{L}\right)\right] \nonumber\\
&= \identity - 2\eta \mathbb{E}_{(s,a)\sim\mu} \left[\sum_{l=1}^{L} \phi_{l} \phi^\top_{l}\right] + \mathbb{E}_{(s,a)\sim\mu} \left[\sum_{k=2}^{2L} (-\eta)^k \sum_{l_1,\ldots,l_k} \phi_{l_1} \phi^\top_{l_1} \ldots \phi_{l_k} \phi^\top_{l_k}\right].
\end{align}
In the third term on the right-hand side (RHS) of the second line, the summation is over all valid combinations of the indices $(l_1, l_2, \ldots, l_k)$, where $l_1, l_2, \ldots, l_k \in \{1, 2, \ldots, L\}$. This is determined by first selecting the index $l_1$ from the index sequence $[L, L-1, \ldots, 2, 1, 1, 2, \ldots, L-1, L]$, as seen in the first row of the equation above. The second index $l_2$ is then chosen, ensuring that $l_2$ lies to the right of $l_1$. The valid combination constraint requires the entire sequence $l_1, \ldots, l_k$ to satisfy the condition that $l_{i-1}$ must appear to the left of $l_{i}$.

The main challenge to upper-bound the entire product $\Gamma_{L}^\top \Gamma_{L}$ under expectation lies in upper-bound the combinatorially many high-order terms. Our approach leverages the high-level idea that the RHS of Equation \eqref{eq:gamma-gamma} can be upper-bounded by a form of $\mathbb{E}_{(s,a)\sim\mu} \left[\sum_{l=1}^{L}\phi_{l}\phi^\top_{l}\right]$ with an appropriate coefficient. Specifically, we demonstrate that the third term on the RHS, which contains a large number of combinatorial terms of the form $\phi_{l_1} \phi^\top_{l_1} \cdots \phi_{l_k} \phi^\top_{l_k}$, can be bounded by terms involving only $\phi_{l} \phi^\top_{l}$ (with $1 \le l \le L$) through the use of a proposed combinatorial counting method.

\begin{theorem} \label{thm:main}
Let $\mu$ be the stationary distribution of the state-action pair in the MDP. The following matrix inequalities, which are positive semi-definite, hold for $\eta \in (0,1)$:
\begin{align} 
\mathbb{E}_{(s,a)\sim\mu} \left[\Gamma_{L}^\top \Gamma_{L}\right] \preceq \left(1 - \frac{ (\eta (4-2L)  -(1 - \eta)^{L-1} -\eta^2+1)L}{\kappa}\right)\identity,
\end{align}
where the matrix $\Gamma_L$ is defined in Definition~\ref{def:gamma}. The relation $\preceq$ between the matrices on both sides is defined in Definition~\ref{def:psd}, referring to the positive semi-definite property.
\end{theorem}

\begin{proof}[Proof Sketch] 
By the linearity of expectation, we can upper-bound the second part of Equation~\eqref{eq:gamma-gamma} as follows:
\begin{align}
-2\eta \mathbb{E}_{(s,a)\sim\mu} \left[\sum_{l=1}^{L} \phi_{l} \phi^\top_{l}\right] = -2\eta \sum_{l=1}^{L} \mathbb{E}_{(s,a)\sim\mu} \left[\phi_{l} \phi^\top_{l}\right] = -2\eta L \mathbb{E}_{(s,a)\sim\mu} \left[\phi \phi^\top\right] \preceq -\frac{2\eta L}{\kappa} \identity.
\end{align}
Based on the new analysis from Lemma~\eqref{lem:combi-cases}, the third part in Equation~\eqref{eq:gamma-gamma} is upper-bounded as:
\begin{align*}
\mathbb{E}_{(s,a)\sim\mu} \left[\sum_{k=2}^{2L} (-\eta)^k \sum_{l_1,\ldots,l_k} \phi_{l_1} \phi^\top_{l_1} \ldots \phi_{l_k} \phi^\top_{l_k} \right] 
&\preceq \mathbb{E}_{(s,a)\sim\mu} \left[\sum_{k=2}^{2L} (-\eta)^k \sum_{l_1,\ldots,l_k} \frac{1}{2}( \phi_{l_1} \phi^\top_{l_1} + \phi_{l_k} \phi^\top_{l_k} ) \right]  \\
&\preceq \left( (1 - \eta)^{L-1} + \eta^2 + \eta(2L-2) - 1 \right) \mathbb{E}_{(s,a)\sim\mu} \left[\sum_{l=1}^L \phi_{l} \phi^\top_{l}\right] \\
&\preceq \frac{((1 - \eta)^{L-1} + \eta^2  + \eta(2L-2) - 1)L}{\kappa} \identity.
\end{align*}
Combining these two inequalities, we arrive at the upper bound stated in the theorem. A detailed proof can be found in Appendix~\ref{apx:combi-count-proof}.
\end{proof}

Theorem~\ref{thm:main} is established based on the new analysis in Lemma~\eqref{lem:combi-cases}, which is introduced in Section~\ref{sec:comb-count}. It serves as a key component in the final convergence proof of the RER algorithm, which will be presented in Section~\ref{sec:final-theorem}.

\paragraph{Numerical Justification of the Tighter Bound}
We provide a numerical evaluation of the derived bound and the original bound in~\citet[Lemma 8]{DBLP:conf/iclr/Agarwal2022} in Figure~\ref{fig:num-comp}\footnote{The code implementation for the numerical evaluation of the equalities and inequalities in this paper is available at \url{https://github.com/jiangnanhugo/RER-proof}.}. For a fixed value of sequence length $L$, we compare the value $(\eta (4 - 2L) - (1 - \eta)^{L-1} - \eta^2 + 1) L$ in our derived upper bound and the original value $\eta L$. For all the different sequence lengths, our derived expression value is numerically higher than the original expression, which implies our bound (in Lemma~\ref{lem:combi-weighted}) is tighter than the original one in~\citet[Lemma 8]{DBLP:conf/iclr/Agarwal2022}.

\begin{figure}[!h]
    \centering
    \includegraphics[width=0.99\linewidth]{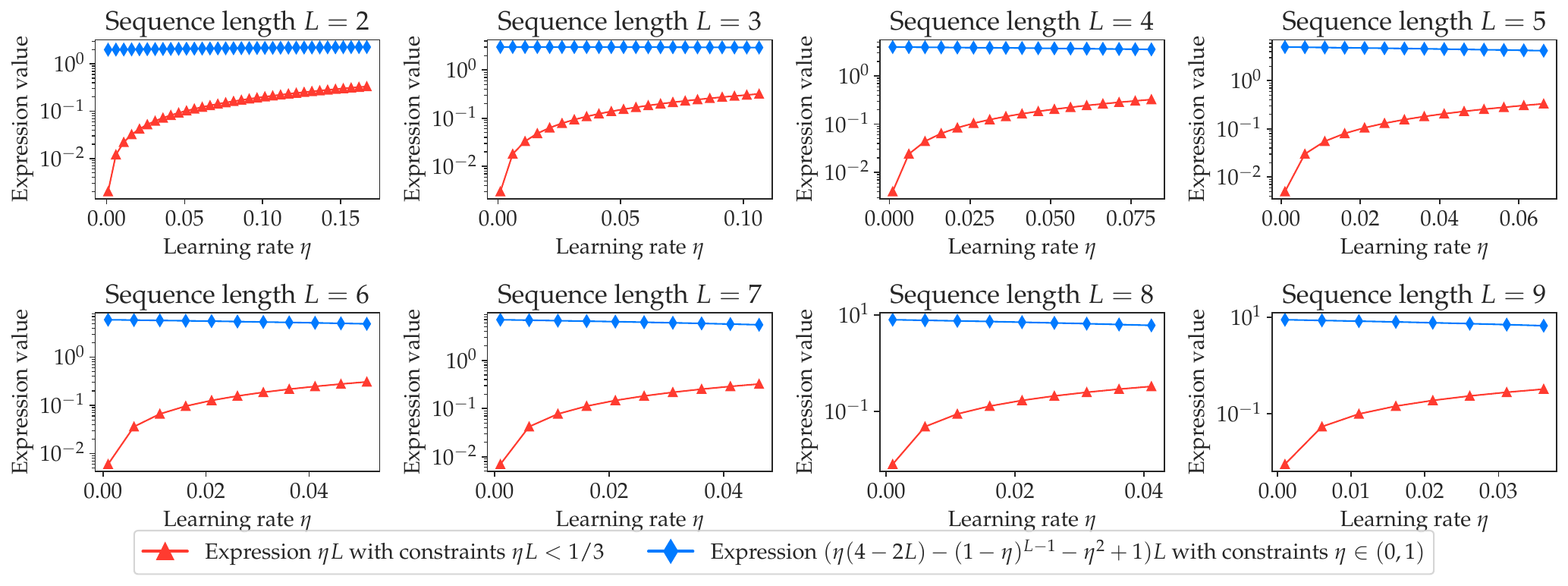}
       \caption{For all the different sequence lengths, our derived expression value is numerically higher than the original expression, which implies our bound (in Lemma~\ref{lem:combi-weighted}) is tighter than the original one in~\citet[Lemma 8]{DBLP:conf/iclr/Agarwal2022}.}
    \label{fig:num-comp}
\end{figure}

\subsection{Relaxing the Requirement $\eta L\le 1/3$ through Combinatorial Counting} \label{sec:comb-count}

\begin{lemma}\label{lem:relax}
Let $\mathbf{x} \in \mathbb{R}^d$ be any non-zero $d$-dimensional vector. For $l_1, \ldots, l_k \in \{1, 2, \ldots, L\}$ and $2 \leq k \leq 2L$, consider a high-order term $\phi_{l_1} \phi^\top_{l_1} \ldots \phi_{l_k} \phi^\top_{l_k}$ in Equation~\eqref{eq:gamma-gamma}. By Assumption~\ref{rem:linear-expect}, we can relax this high-order term as follows:
\begin{align}
|\mathbf{x}^\top \phi_{l_1} \phi^\top_{l_1} \ldots \phi_{l_k} \phi^\top_{l_k} \mathbf{x}| &\le \frac{1}{2} \mathbf{x}^\top \left( \phi_{l_1} \phi^\top_{l_1} + \phi_{l_k} \phi^\top_{l_k} \right) \mathbf{x}.
\end{align}
\end{lemma}

The proof of this inequality can be found in Appendix~\ref{apx-relax-proof}. 

This result implies that, after relaxation, only the first term $\phi_{l_1} \phi^\top_{l_1}$ (indexed by $l_1$) and the last term $\phi_{l_k} \phi^\top_{l_k}$ (indexed by $l_k$) determine the upper bound of the high-order term $\phi_{l_1} \phi^\top_{l_1} \ldots \phi_{l_k} \phi^\top_{l_k}$. This relaxation simplifies the original complex summation problem $\sum_{1 \leq l_1, \ldots, l_k \leq L}$ to count how many valid $l_1$ and $l_k$ can be selected at each possible position in the sequence of transitions.

\begin{lemma} \label{lem:combi-cases}
Based on the relaxation provided in Lemma~\ref{lem:relax}, the third part
in Equation~\eqref{eq:gamma-gamma} can be expanded combinatorially as follows:
 \begin{equation} \label{eq:comb}
\sum_{k=2}^{2L} (-\eta)^k \sum_{l_1, \ldots, l_k} \frac{1}{2} (\phi_{l_1} \phi^\top_{l_1} + \phi_{l_k} \phi^\top_{l_k}) = \underbrace{\sum_{k=2}^{2L} (-\eta)^k \sum_{l=1}^L \left(\binom{L + l - 2}{k - 1} + \binom{L - l}{k - 1} + \binom{2l - 2}{k - 2}\right)}_{\text{sum over combinatorially many terms}} \phi_l \phi^\top_l
 \end{equation}  
 \begin{proof}[Sketch of Proof]
 As depicted in Figure~\ref{fig:combination-short}, we consider two arrays of length $L$. The indices in these arrays are symmetrical: the left array decreases from $L$ to $1$, while the right array increases from $1$ to $L$. These arrays represent the indices of the matrix products in the first line of Equation~\eqref{eq:gamma-gamma}. The left array simulates \(\Gamma_L\), and the right array simulates \(\Gamma_L^\top\). 
The key idea is to count the number of combinations of $l_1$ and $l_k$ that can produce \(\phi_l \phi_l^\top\) for a fixed $l$ (where $1 \le l \le L$).

\begin{figure}[!h]
    \centering
    \includegraphics[width=1\linewidth]{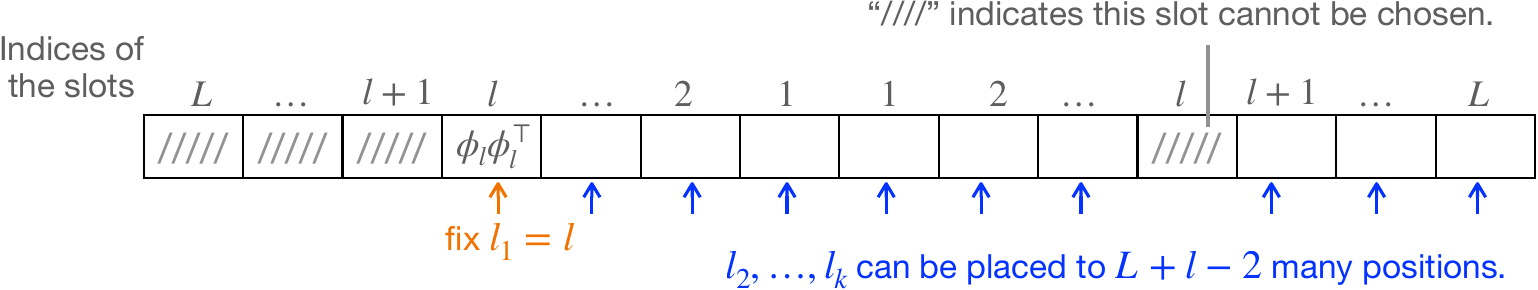}
    \caption{Case 1 in the proposed combinatorial counting procedure. This case illustrates how many terms of the form \(\phi_{l_1} \phi^\top_{l_1} \ldots \phi_{l_k} \phi^\top_{l_k}\) can be reduced to \(\phi_l \phi_l^\top\) for a fixed $l$ using Lemma~\ref{lem:relax}, where $1 \leq l \leq L$. If $l_1$ is assigned to the left $l$-th slot, then $l_k$ cannot choose any of the left terms with indices $L, \ldots, l+1$ due to the sequential ordering constraint $l_i$ must be to the right of $l_{i-1}$. To avoid double counting, $l_k$ is also disallowed from occupying the right $l$-th slot. Consequently, there are $ L + l - 2$ available slots for assigning the remaining sequence $l_2, \ldots, l_k$ of length $k-1$. Therefore, there are \(\binom{L + l - 2}{k - 1}\) such terms for this case. Further cases are illustrated in Figure~\ref{fig:combination} in the appendix.}
    \label{fig:combination-short}
\end{figure}

In the first case, illustrated in Figure~\ref{fig:combination-short}, we fix $l_1$ in the left $l$-th slot. For $l_k$, it cannot choose any of the slots in the left array with indices $L, \ldots, l+1$ due to the sequential ordering constraint, which requires that $l_{i-1}$ must be to the left of $l_i$. Additionally, to avoid double counting, we also exclude the right $l$-th slot for $l_k$. Consequently, there are $L + l - 2$ available slots for assigning the remaining sequence $l_2, \ldots, l_k$. This results in \(\binom{L + l - 2}{k - 1}\) contributions for this case, as shown on the right-hand side.

For the remaining cases, detailed in Figure~\ref{fig:combination} and analyzed in Appendix~\ref{apx:combi-case-proof}, they contribute to the second and last terms in Equation~\eqref{eq:comb}.
\end{proof}
\end{lemma}
Lemma~\ref{lem:combi-cases} demonstrates the process of simplifying the complex summation \(\sum_{l_1, \ldots, l_k}\) into a more manageable form \(\sum_{l=1}^L\). This transformation significantly simplifies the task of obtaining a tighter upper bound.

\begin{lemma}   \label{lem:combi-weighted}
For $\eta \in (0,1)$ and $L > 1$, the following holds: \\$\sum_{k=2}^{2L}(-\eta)^k\left(\binom{L+l-2}{k-1} +\binom{L-l}{k-1}+\binom{2l-2}{k-2}\right)
  =(1 - \eta)^{L+l-2} +  (1 - \eta)^{L-l}  +  \eta^2 (1 - \eta)^{2l-2} +\eta(2L-2) -2$.
\end{lemma}

The proof of Lemma~\ref{lem:combi-weighted} is presented in detail in Appendix~\ref{sec:combi-weighted-proof}, where we utilize the Binomial theorem. To ensure that the oscillatory term $(-\eta)^k$ does not cause divergence, we require the learning rate $\eta$ to lie within the interval $\eta \in (0,1)$.

\section{Sample Complexity of Reverse Experience Replay-Based $Q$-Learning on Linear MDPs} \label{sec:final-theorem}

The convergence analysis assumes that every sub-trajectory of length $L$ is almost (or asymptotically) independent of each other with high probability. This condition, known as the mixing requirement for Markovian data, implies that the statistical dependence between two sub-trajectories $\tau_L$ and $\tau'_L$ diminishes as they become further apart along the trajectory~\citep{DBLP:conf/icml/TagortiS15,DBLP:conf/nips/NagarajWB0N20}.

Prior work~\citep{DBLP:conf/nips/LeeCC19} provided a convergence proof for the Reverse Experience Replay (RER) approach but did not address the rate of convergence, primarily due to the challenges associated with quantifying deep neural networks. By contrast, Linear MDPs (defined in Definition~\ref{def:Linear}), which approximate the reward function and transition kernel linearly via features, allow for an asymptotic performance analysis of RER. Recently, \citet{DBLP:conf/iclr/Agarwal2022} presented the first theoretical proof for RER. However, their analysis is limited by stringent conditions, notably requiring a minimal learning rate $\eta L \le \frac{1}{3}$. This constraint suggests that RER may struggle to compete with plain Experience Replay (ER) when using larger learning rates.

To address this challenge, we provide a tighter theoretical analysis of the RER method in Theorem~\ref{thm:main}. Our analysis mitigate the constraints on the learning rate for convergence. We demonstrate that the convergence rate can be improved with a larger learning rate and a longer sequence of state-action-reward tuples, thus bridging the gap between theoretical convergence analysis and empirical learning results.

\begin{algorithm}[!t]
\caption{Episodic Q-learning with Reverse Experience Replay} \label{algo:RER}
\begin{algorithmic}[1]
\Require{Sequence length $L$ of consecutive state-action tuples; Replay buffer $\mathcal{B}$; Total learning episodes $T$; Target network update frequency $N$.}
\Ensure{The best-learned policy.}
\For{$t= 1$ \textbf{to} $T$}
    % \State \texttt{\# Exploration}
    \State Act by  $\epsilon$-greedy strategy \textit{w.r.t.} policy $\pi$. %\Comment{Exploration}
     \State Save the new trajectory into the replay buffer $\mathcal{B}$.
     % \State \texttt{\# Exploitation}
     \State Retrieve a sub-trajectory $\tau_{L}$ from buffer $\mathcal{B}$, where $\tau_{l}:=(s_{l},a_{l},r_{l})$,  for all $ 1\le l\le L$.
     % \State $Q(w_{0})\leftarrow Q(w_{L+1})$ 
     \For{$l=1$ \textbf{to} $ L$} \Comment{reverse experience replay}
     \State ${\varepsilon\leftarrow r_{L-l}+\gamma\max_{a'\in\mathcal{A}}Q(s_{L+1-l},a';\theta_k)-Q_{L+1-l}}$ 
      \State
     ${w_{t,l+1}\leftarrow w_{t,l}+\eta\varepsilon\nabla Q_{t,L+1-l}}$
     \EndFor
     \If{ $t\mod N =0$}  \Comment{online target update}
     \State $\theta_{k}\leftarrow w_{t,L+1}$ 
     \State $k\leftarrow k+1$
     \EndIf
    % \State \texttt{\# Policy Extraction}
    \State $\pi(s)\leftarrow\arg\max_{a\in\mathcal{A}},Q(s,a;w_{t,L+1})$, for all $s\in\mathcal{S}$. \Comment{policy extraction}
\EndFor
\State \textbf{Return} The converged policy $\pi$. 
\end{algorithmic}
\end{algorithm}

\begin{lemma}[Bias and variance decomposition] \label{lemma:error-decomp}
Let the error terms for every parameter $w$ as the difference between empirical estimation and true MDP: $\varepsilon_i(w)\coloneqq {Q}(s_i,a_i)-Q^*(s_i,a_i)$. For the current iteration $t$, the  difference between current estimated parameter $w$ and the optimal parameter $w^*$ accumulated along the $L$ length transitions with reverse update is:
\begin{align}
w_{L}-w^*=\underbrace{\Gamma_{L}\left(w_{1}-w^*\right)}_{\text{Bias term}}+\underbrace{\eta\sum_{l=1}^L \varepsilon_l\Gamma_{l-1}\phi_l}_{\text{variance term}}.
\end{align}
For clarity, $\Gamma_L$ in Definition~\ref{def:gamma} is a joint product of $L$ terms involving the feature vector of the consecutive state-action tuples.  When the norm of $\Gamma_L$ is small, the parameter will quickly converge to its optimal.

The first part on RHS is noted as the bias and the second part on RHS is variance along the sub-trajectory, which we will later show with zero mean.
\end{lemma}
The proof is presented in Appendix~\ref{sec:error-decompose}. The result is obtained by unrolling the terms for consecutive $L$ steps in reverse update order according to Lines 5-7 in Algorithm~\ref{algo:RER}. This allows us to separately quantify the upper bound the bias term and the variance terms.

\begin{lemma}[Bound on the bias term]\label{lem:contract2}
Let $\mathbf{x}\in \mathbb{R}^d$ be a non-zero vector and $N$ is the frequency for the target network to be updated. For $\eta\in(0,1)$ and $L >1$,   the following matrix's positive semi-definite inequality holds with probability at least $1-\delta$:
\begin{align}
    \mathbb{E}\norm{\prod^{1}_{j=N}\Gamma_{L}\mathbf{x}}_{\phi}^2&\le \exp\left(-\frac{(\eta (4-2L)  -\eta^2+ 1)NL}{\kappa}\right)\sqrt{\tfrac{\kappa}{\delta}}\norm{\mathbf{x}}_{\phi}.
\end{align}
The $\phi$-based norm is defined in Definition~\ref{def:phi}.
\begin{proof}[Sketch of proof] The result is obtained first expand the joint product over $\prod_{j=N}^i$ over $\Gamma_L$ and integrate the result in Theorem~\ref{thm:main}. The detailed proof is presented in Appendix~\ref{apx:bias-term-bound}.
\end{proof}
\end{lemma}

In terms of the bound for the variance term in Lemma~\ref{lemma:error-decomp}, even though the term $\Gamma_l$ is involved in the expression, it turns out we do not need to modify the original proof and thus we follow the result in the original work. The exact statement is presented in the Appendix~\ref{apx:variance-bound}.

\begin{theorem}\label{thm:main-cvg}
For Linear MDP, assume the reward function, as well as the feature, is bounded $R(s,a)\in[0,1]$, $\|\phi(s,a)\|_2\le 1$, for all $(s,a)\in\mathcal{S}\times\mathcal{A}$. Let $T$ be the maximum learning episodes, $N$ be the frequency of the target network update, $\eta$ be the learning rate and $L$ be the length of sequence for RER described in Algorithm~\ref{algo:RER}. When $\eta\in(0,1), L\ge 1$, with sample complexity
\begin{align}
\mathcal{O}\left(\frac{\gamma^{T/N}}{1-\gamma}+\sqrt{\frac{T\kappa}{N\delta(1-\gamma)^4}}\textcolor{blue}{\exp\left(-\frac{(\eta (4-2L)   -\eta^2+1)NL}{\kappa}\right)}+\sqrt{\frac{\eta \log(\frac{T}{N\delta})}{(1-\gamma)^4}}\right),
\end{align}
$\|Q_T(s,a)-Q^*(s,a)\|_{\infty}\le\varepsilon$  holds with probability at least $1-\delta$.
\begin{proof}[Sketch of Proof]
We first establish the independence of sub-trajectories of length $L$. We then decompose the error term of the $Q$-value using bias-variance decomposition (as shown in Lemma~\ref{lemma:error-decomp}), where the RER method and target network help control the variance term using martingale sequences. The upper bound for the bias term is given in Lemma~\ref{lem:contract2} and the upper bound for the variance term is presented in Lemma~\ref{lem:linear_variance}. Finally, we summarize the results and provide the complete proof in Lemma~\ref{lem:final}, leading to the probabilistic bound in this theorem.
\end{proof}
\end{theorem}

Compared to the original theorem in \citet[Theorem 1]{DBLP:conf/iclr/Agarwal2022}, our work provides a tighter upper bound and relaxes the assumptions needed for the result to hold. This advancement bridges the gap between theoretical justification and empirical MDP evaluation. Furthermore, we hope that the new approach of transforming the original problem into a combinatorial counting problem will inspire further research in related domains.

We acknowledge that the main structure of the convergence proof (i.e., Theorem~\ref{thm:main-cvg}) follows the original work. Our contribution lies in presenting a cleaner proof pipeline and incorporating our tighter bound as detailed in Theorem~\ref{thm:main}.

% \section{Experimental Analysis}
% In this section, we demonstrate the derived upper bound is tight and has no limitation on the sequence length and learning rate constraint (in Figure~\ref{fig:error-surface1} and Figure~\ref{fig:error-surface2} ). Furthermore, we show that the algorithm can still stably converge to the optimal when for large learning rate and a long sequence of state action reward tuples.

% \subsection{Large Learning Rate improve empirical convergence}

% \subsection{Longer consecutive sequence improve empirical convergence}

% [TODO:] evaluate by breaking this constraint. with a large learning rate.
\section{Conclusion}

In this work, we gave a tighter finite-sample analysis for heuristics which are heavily used in practical
$Q$-learning and showed that seemingly simple modifications can have far-reaching consequences
in linear MDP settings.  We provide a rigorous analysis that relaxes the limitation on the learning rate and the length of the consecutive tuples. Our key idea is to transform the original problem involving a giant summation into a combinatorial counting problem, which greatly simplifies the whole problem.
Finally, we show theoretically that RER converges faster with a larger learning rate $\eta$ and a longer consecutive sequence $L$ of state-action-reward tuples.

\subsubsection*{Acknowledgments}
\label{sec:ack}
We thank all the reviewers for their constructive comments. We also thank Yi Gu for his feedback on the theoretical justification part of this paper.
This research was supported by NSF grant CCF-1918327 and DOE – Fusion Energy Science grant: DE-SC0024583.
%% The file named.bst is a bibliography style file for BibTeX 0.99c

%%%%%%%%%%%%%%%%%%%%%%%%%%%%%%%%%%%%%%%%%%%%%%%%%%%%%%%%%%%%%%%%
%% NOTE: THIS MARKS THE END OF THE "MAIN TEXT"
%%%%%%%%%%%%%%%%%%%%%%%%%%%%%%%%%%%%%%%%%%%%%%%%%%%%%%%%%%%%%%%%

%%%%%%%%%%%%%%%%%%%%%%%%%%%%%%%%%%%%%%%%%%%%%%%%%%%%%%%%%%%%%%%%
%% Bibliography
%%%%%%%%%%%%%%%%%%%%%%%%%%%%%%%%%%%%%%%%%%%%%%%%%%%%%%%%%%%%%%%%
\bibliography{reference}

\begin{thebibliography}{28}
\providecommand{\natexlab}[1]{#1}
\providecommand{\url}[1]{\texttt{#1}}
\expandafter\ifx\csname urlstyle\endcsname\relax
  \providecommand{\doi}[1]{doi: #1}\else
  \providecommand{\doi}{doi: \begingroup \urlstyle{rm}\Url}\fi

\bibitem[Agarwal et~al.(2022)Agarwal, Chaudhuri, Jain, Nagaraj, and
  Netrapalli]{DBLP:conf/iclr/Agarwal2022}
Naman Agarwal, Syomantak Chaudhuri, Prateek Jain, Dheeraj Nagaraj, and Praneeth
  Netrapalli.
\newblock Online target q-learning with reverse experience replay: Efficiently
  finding the optimal policy for linear mdps.
\newblock In \emph{{ICLR}}. OpenReview.net, 2022.

\bibitem[Ambrose et~al.(2016)Ambrose, Pfeiffer, and Foster]{AMBROSE20161124}
R.~Ellen Ambrose, Brad~E. Pfeiffer, and David~J. Foster.
\newblock Reverse replay of hippocampal place cells is uniquely modulated by
  changing reward.
\newblock \emph{Neuron}, 91\penalty0 (5):\penalty0 1124--1136, 2016.
\newblock ISSN 0896-6273.

\bibitem[Andrychowicz et~al.(2017)Andrychowicz, Crow, Ray, Schneider, Fong,
  Welinder, McGrew, Tobin, Abbeel, and
  Zaremba]{DBLP:conf/nips/AndrychowiczCRS17}
Marcin Andrychowicz, Dwight Crow, Alex Ray, Jonas Schneider, Rachel Fong, Peter
  Welinder, Bob McGrew, Josh Tobin, Pieter Abbeel, and Wojciech Zaremba.
\newblock Hindsight experience replay.
\newblock In \emph{{NIPS}}, pp.\  5048--5058, 2017.

\bibitem[Bai et~al.(2021)Bai, Wang, Han, Hao, Garg, Liu, and
  Wang]{DBLP:conf/icml/BaiWHHG0W21}
Chenjia Bai, Lingxiao Wang, Lei Han, Jianye Hao, Animesh Garg, Peng Liu, and
  Zhaoran Wang.
\newblock Principled exploration via optimistic bootstrapping and backward
  induction.
\newblock In \emph{{ICML}}, volume 139 of \emph{Proceedings of Machine Learning
  Research}, pp.\  577--587. {PMLR}, 2021.

\bibitem[Bertsekas \& Yu(2012)Bertsekas and Yu]{DBLP:journals/mor/BertsekasY12}
Dimitri~P. Bertsekas and Huizhen Yu.
\newblock Q-learning and enhanced policy iteration in discounted dynamic
  programming.
\newblock \emph{Math. Oper. Res.}, 37\penalty0 (1):\penalty0 66--94, 2012.

\bibitem[Florensa et~al.(2017)Florensa, Held, Wulfmeier, Zhang, and
  Abbeel]{DBLP:conf/corl/FlorensaHWZA17}
Carlos Florensa, David Held, Markus Wulfmeier, Michael Zhang, and Pieter
  Abbeel.
\newblock Reverse curriculum generation for reinforcement learning.
\newblock In \emph{CoRL}, volume~78 of \emph{Proceedings of Machine Learning
  Research}, pp.\  482--495. {PMLR}, 2017.

\bibitem[Foster \& Wilson(2006)Foster and Wilson]{foster2006reverse}
David~J Foster and Matthew~A Wilson.
\newblock Reverse replay of behavioural sequences in hippocampal place cells
  during the awake state.
\newblock \emph{Nature}, 440\penalty0 (7084):\penalty0 680--683, 2006.

\bibitem[Goyal et~al.(2019)Goyal, Brakel, Fedus, Singhal, Lillicrap, Levine,
  Larochelle, and Bengio]{DBLP:conf/iclr/GoyalBFSLLLB19}
Anirudh Goyal, Philemon Brakel, William Fedus, Soumye Singhal, Timothy~P.
  Lillicrap, Sergey Levine, Hugo Larochelle, and Yoshua Bengio.
\newblock Recall traces: Backtracking models for efficient reinforcement
  learning.
\newblock In \emph{{ICLR}}. OpenReview.net, 2019.

\bibitem[Haddad(2021)]{haddad2021repeated}
Roudy~El Haddad.
\newblock Repeated sums and binomial coefficients.
\newblock \emph{arXiv preprint arXiv:2102.12391}, 2021.

\bibitem[Igata et~al.(2021)Igata, Ikegaya, and
  Sasaki]{doi:10.1073/pnas.2011266118}
Hideyoshi Igata, Yuji Ikegaya, and Takuya Sasaki.
\newblock Prioritized experience replays on a hippocampal predictive map for
  learning.
\newblock \emph{Proceedings of the National Academy of Sciences}, 118\penalty0
  (1), 2021.

\bibitem[Jain et~al.(2021)Jain, Kowshik, Nagaraj, and
  Netrapalli]{DBLP:conf/nips/Jain2021}
Prateek Jain, Suhas~S. Kowshik, Dheeraj Nagaraj, and Praneeth Netrapalli.
\newblock Streaming linear system identification with reverse experience
  replay.
\newblock In \emph{{NIPS}}, volume~34, pp.\  30140--30152. Curran Associates,
  Inc., 2021.

\bibitem[Kumar et~al.(2020)Kumar, Gupta, and Levine]{DBLP:conf/nips/Kumar0L20}
Aviral Kumar, Abhishek Gupta, and Sergey Levine.
\newblock Discor: Corrective feedback in reinforcement learning via
  distribution correction.
\newblock In \emph{NeurIPS}, volume~33, pp.\  18560--18572, 2020.

\bibitem[Lahire et~al.(2022)Lahire, Geist, and
  Rachelson]{DBLP:conf/icml/LahireGR22}
Thibault Lahire, Matthieu Geist, and Emmanuel Rachelson.
\newblock Large batch experience replay.
\newblock In \emph{{ICML}}, volume 162 of \emph{Proceedings of Machine Learning
  Research}, pp.\  11790--11813. {PMLR}, 2022.

\bibitem[Lee et~al.(2019)Lee, Choi, and Chung]{DBLP:conf/nips/LeeCC19}
Su~Young Lee, Sung{-}Ik Choi, and Sae{-}Young Chung.
\newblock Sample-efficient deep reinforcement learning via episodic backward
  update.
\newblock In \emph{NeurIPS}, pp.\  2110--2119, 2019.

\bibitem[Li et~al.(2022)Li, Wei, Chi, Gu, and Chen]{DBLP:journals/tit/LiWCGC22}
Gen Li, Yuting Wei, Yuejie Chi, Yuantao Gu, and Yuxin Chen.
\newblock Sample complexity of asynchronous q-learning: Sharper analysis and
  variance reduction.
\newblock \emph{{IEEE} Trans. Inf. Theory}, 68\penalty0 (1):\penalty0 448--473,
  2022.

\bibitem[Lin(1992)]{DBLP:journals/ml/Lin92}
Long~Ji Lin.
\newblock Self-improving reactive agents based on reinforcement learning,
  planning and teaching.
\newblock \emph{Mach. Learn.}, 8:\penalty0 293--321, 1992.

\bibitem[Mnih et~al.(2015)Mnih, Kavukcuoglu, Silver, Rusu, Veness, Bellemare,
  Graves, Riedmiller, Fidjeland, Ostrovski, Petersen, Beattie, Sadik,
  Antonoglou, King, Kumaran, Wierstra, Legg, and Hassabis]{nature2015dqn}
Volodymyr Mnih, Koray Kavukcuoglu, David Silver, Andrei~A. Rusu, Joel Veness,
  Marc~G. Bellemare, Alex Graves, Martin Riedmiller, Andreas~K. Fidjeland,
  Georg Ostrovski, Stig Petersen, Charles Beattie, Amir Sadik, Ioannis
  Antonoglou, Helen King, Dharshan Kumaran, Daan Wierstra, Shane Legg, and
  Demis Hassabis.
\newblock Human-level control through deep reinforcement learning.
\newblock \emph{Nature}, 518\penalty0 (7540):\penalty0 529--533, 2015.

\bibitem[Nagaraj et~al.(2020)Nagaraj, Wu, Bresler, Jain, and
  Netrapalli]{DBLP:conf/nips/NagarajWB0N20}
Dheeraj Nagaraj, Xian Wu, Guy Bresler, Prateek Jain, and Praneeth Netrapalli.
\newblock Least squares regression with markovian data: Fundamental limits and
  algorithms.
\newblock In \emph{NeurIPS}, 2020.

\bibitem[Puterman(1994)]{DBLP:books/wi/Puterman94}
Martin~L. Puterman.
\newblock \emph{Markov Decision Processes: Discrete Stochastic Dynamic
  Programming}.
\newblock Wiley Series in Probability and Statistics. Wiley, 1994.

\bibitem[Rotinov(2019)]{DBLP:journals/corr/Rotinov2019}
Egor Rotinov.
\newblock Reverse experience replay.
\newblock \emph{CoRR}, abs/1910.08780, 2019.

\bibitem[Saglam et~al.(2023)Saglam, Mutlu, {\c{C}}i{\c{c}}ek, and
  Kozat]{DBLP:journals/jair/SaglamMCK23}
Baturay Saglam, Furkan~B. Mutlu, Dogan~Can {\c{C}}i{\c{c}}ek, and Suleyman~S.
  Kozat.
\newblock Actor prioritized experience replay.
\newblock \emph{J. Artif. Intell. Res.}, 78:\penalty0 639--672, 2023.

\bibitem[Schaul et~al.(2016)Schaul, Quan, Antonoglou, and
  Silver]{DBLP:journals/corr/SchaulQAS15}
Tom Schaul, John Quan, Ioannis Antonoglou, and David Silver.
\newblock Prioritized experience replay.
\newblock In \emph{{ICLR}}, 2016.

\bibitem[Tagorti \& Scherrer(2015)Tagorti and
  Scherrer]{DBLP:conf/icml/TagortiS15}
Manel Tagorti and Bruno Scherrer.
\newblock On the rate of convergence and error bounds for lstd ($\lambda$).
\newblock In \emph{{ICML}}, volume~37 of \emph{{JMLR} Workshop and Conference
  Proceedings}, pp.\  1521--1529. JMLR.org, 2015.

\bibitem[van Hasselt et~al.(2016)van Hasselt, Guez, and
  Silver]{DBLP:conf/aaai/HasseltGS16}
Hado van Hasselt, Arthur Guez, and David Silver.
\newblock Deep reinforcement learning with double q-learning.
\newblock In \emph{{AAAI}}, pp.\  2094--2100. {AAAI} Press, 2016.

\bibitem[Vershynin(2018)]{vershynin2018high}
Roman Vershynin.
\newblock \emph{High-dimensional probability: An introduction with applications
  in data science}, volume~47.
\newblock Cambridge university press, 2018.

\bibitem[Watkins \& Dayan(1992)Watkins and Dayan]{DBLP:journals/ml/WatkinsD92}
Christopher J. C.~H. Watkins and Peter Dayan.
\newblock Technical note q-learning.
\newblock \emph{Mach. Learn.}, 8:\penalty0 279--292, 1992.

\bibitem[Watkins(1989)]{watkins1989learning}
Christopher John Cornish~Hellaby Watkins.
\newblock Learning from delayed rewards.
\newblock \emph{PhD thesis, King's College, University of Cambridge}, 1989.

\bibitem[Zanette et~al.(2020)Zanette, Lazaric, Kochenderfer, and
  Brunskill]{DBLP:conf/icml/ZanetteLKB20}
Andrea Zanette, Alessandro Lazaric, Mykel~J. Kochenderfer, and Emma Brunskill.
\newblock Learning near optimal policies with low inherent bellman error.
\newblock In \emph{{ICML}}, volume 119 of \emph{Proceedings of Machine Learning
  Research}, pp.\  10978--10989. {PMLR}, 2020.

\end{thebibliography}
\bibliographystyle{rlc}

%%%%%%%%%%%%%%%%%%%%%%%%%%%%%%%%%%%%%%%%%%%%%%%%%%%%%%%%%%%%%%%%
%% Appendices
%%%%%%%%%%%%%%%%%%%%%%%%%%%%%%%%%%%%%%%%%%%%%%%%%%%%%%%%%%%%%%%%
\appendix
\clearpage

\section{Proof of Combinatorial Counting in Reverse Experience Replay} \label{apx:comb-count}
In this section, we present detailed proof for each of the necessary lemmas used for Theorem~\ref{thm:main}.

\subsection{Proof of Lemma~\ref{lem:relax}} \label{apx-relax-proof}

\begin{lemma-no}
Let $\mathbf{x}\in\mathbb R^d$ be any $d$-dimensional non-zero vector, i.e., $\mathbf{x}\neq \mathbf{0}$. For $1 \le l_1,\ldots, l_k \le L$ and $2\le k\le 2L$,  consider one high-order term $\phi_{l_1}\phi^\top_{l_1}\ldots\phi_{l_k}\phi^\top_{l_k}$ in Equation~\eqref{eq:gamma-gamma}.  We have:
\begin{align}
|\mathbf{x}^\top\phi_{l_1}\phi^\top_{l_1}\ldots\phi_{l_k}\phi^\top_{l_k}\mathbf{x}|&\le  \frac{1}{2}\mathbf{x}^\top\left( \phi_{l_1}\phi^\top_{l_1}+\phi_{l_k}\phi^\top_{l_k} \right)\mathbf{x},
\end{align}
where $|\cdot|$ computes the absolute value.
\begin{proof}
For $1 \le l_1,\ldots, l_k \le L$ and $2\le k\le 2L$, we consider one high-order term $\phi_{l_1}\phi^\top_{l_1}\ldots\phi_{l_k}\phi^\top_{l_k}$ in Equation~\eqref{eq:gamma-gamma}.  By the Cauchy-Schwarz inequality (shown in Lemma~\ref{lem:cauchy}), we have:
\begin{align}\label{eq:neighbor-inner-product}
|\phi^\top_{l_{j}}\phi_{l_{j+1}}|\le 1, \qquad \text{ for } 1\le j< k
\end{align}
This allows us to simplify the neighboring vector-vector inner product in the high-order term $\phi_{l_1}\left(\phi^\top_{l_1}\phi_{l_2}\right)\left(\phi^\top_{l_2}\phi_{l_3}\right)\ldots(\phi_{l_{k-1}}^\top\phi_{l_k})\phi^\top_{l_k}$. 

Let $\mathbf{x}\in\mathbb R^d$ be any $d$-dimensional non-zero vector, i.e., $\mathbf{x}\neq \mathbf{0}$. Based on the above inequality (in Equation~\ref{eq:neighbor-inner-product}), we have:
\begin{align}
|\mathbf{x}^\top\phi_{l_1}\phi^\top_{l_1}\ldots\phi_{l_k}\phi^\top_{l_k}\mathbf{x}|&\le |\mathbf{x}^\top \phi_{l_1} \phi^\top_{l_k} \mathbf{x}|  &\text{by Equation~\eqref{eq:neighbor-inner-product}}\nonumber\\
 & \le \frac{1}{2}\left(\left(\mathbf{x}^\top \phi_{l_1}\right)^2+ \left(\phi^\top_{l_k} \mathbf{x}\right)^2\right)   &\text{AM-GM inequality (in Lemma~\ref{lem:am-gm})}\nonumber\\
  & =\frac{1}{2}\left(\left( \phi_{l_1}^\top\mathbf{x}\right)^2+ \left(\phi^\top_{l_k} \mathbf{x}\right)^2\right)   \nonumber\\
&= \frac{1}{2}\left(\mathbf{x}^\top \phi_{l_1}\phi^\top_{l_1}\mathbf{x}+\mathbf{x}^\top\phi_{l_k}\phi^\top_{l_k}\mathbf{x} \right) \nonumber\\
&= \frac{1}{2}\mathbf{x}^\top\left( \phi_{l_1}\phi^\top_{l_1}+\phi_{l_k}\phi^\top_{l_k} \right)\mathbf{x},
\end{align}
where the third line holds because the inner product between two vectors equals its transpose.
\end{proof}
\end{lemma-no}

Note that the above Lemma~\ref{lem:relax} is extended from~\citep[Claim 4]{DBLP:conf/nips/Jain2021}.

\subsection{Proof of Lemma~\ref{lem:combi-cases}} \label{apx:combi-case-proof}

\begin{figure}[p]
  \centering
  \begin{subfigure}{\linewidth}
  \centering
    \includegraphics[width=\linewidth]{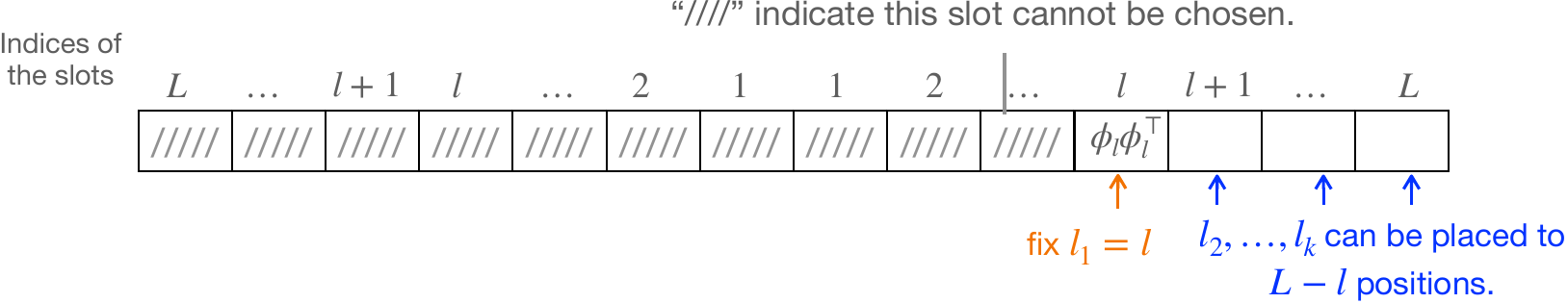}
    \caption{When $l_1$ selects the \textit{right} $l$-th slot, the terms $l_2, \ldots, l_k$ can only choose from the \textit{right} slots with indices $l+1, \ldots, L$, due to the sequential order constraint $l_1 \le l_2 \le \cdots \le l_k$. Thus, there are $L - l$ available slots for $l_2, \ldots, l_k$.}
    \label{fig:combi-case2}
  \end{subfigure}
  \vspace{10pt}\\
  \begin{subfigure}{\linewidth}
  \centering
    \includegraphics[width=0.99\linewidth]{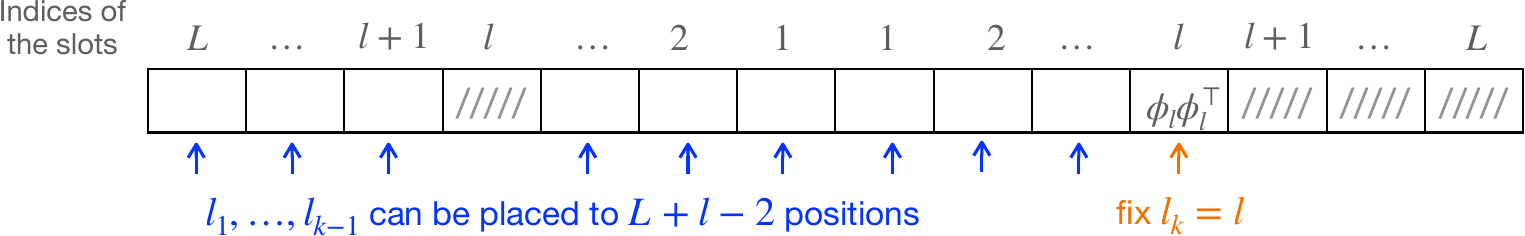}
    \caption{When $l_k$ selects the right $l$-th slot, $l_1, \ldots, l_{k-1}$ cannot choose from all the \textit{right} slots with indices $l+1, \ldots, L$ because of the sequential order constraint. To avoid double counting, the left $l$-th slot is also disallowed. Therefore, there are $L + l - 2$ slots available for $l_1, \ldots, l_{k-1}$.}
    \label{fig:combi-case3}
  \end{subfigure}
  \vspace{10pt}\\
  \begin{subfigure}{\linewidth}
  \centering
    \includegraphics[width=0.99\linewidth]{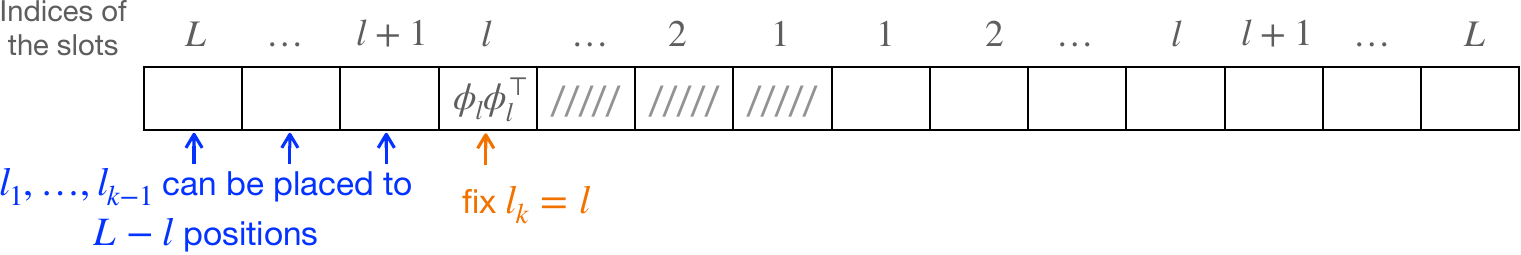}
    \caption{When $l_k$ selects the \textit{left} $l$-th slot, $l_1, \ldots, l_{k-1}$ can only choose from the \textit{left} slots with indices $L, L-1, \ldots, l+1$, due to the sequential order constraint. Similar to case~\ref{fig:combi-case2}, there are $L - l$ slots available for $l_1, \ldots, l_{k-1}$.}
    \label{fig:combi-case4}
  \end{subfigure}
  \vspace{10pt}\\
  \begin{subfigure}{\linewidth}
  \centering
    \includegraphics[width=0.99\linewidth]{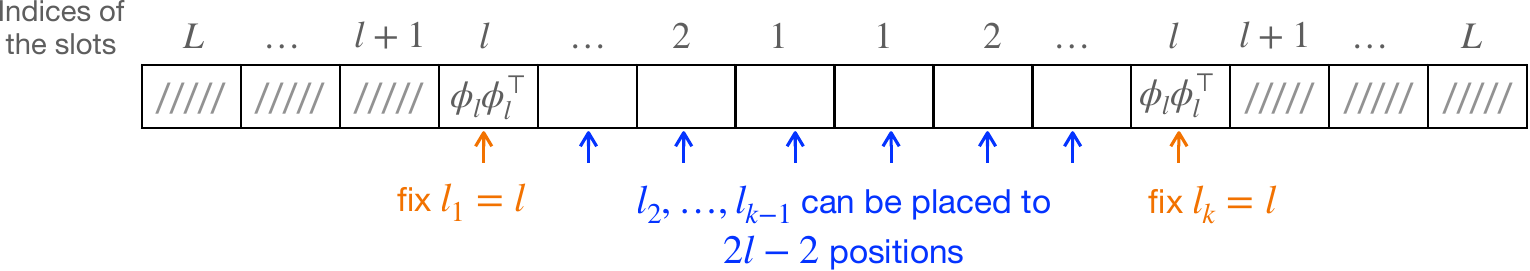}
    \caption{When $l_1$ selects the \textit{left} $l$-th slot and $l_k$ selects the \textit{right} $l$-th slot, the intermediate sequence $l_2, \ldots, l_{k-1}$ can only select from the intermediate slots with indices $l-1, \ldots, 2, 1, 1, 2, \ldots, l-1$. There are $2l - 2$ slots available for placing the $k-2$ intermediate elements.}
    \label{fig:comb-case5}
  \end{subfigure}
  \caption{Visualization of all cases for the combinatorial counting problem.}
  \label{fig:combination}
\end{figure}

\begin{lemma-no} 
Based on the relaxation in Lemma~\ref{lem:relax}, the weighted summation $\sum_{k=2}^{2L}(-\eta)^k\sum_{l_1,\ldots,l_k}$ in Equation~\eqref{eq:gamma-gamma} can be expanded combinatorially as follows:
\begin{align}
\sum_{k=2}^{2L}(-\eta)^k\sum_{l_1,\ldots,l_k}\frac{1}{2}( \phi_{l_1}\phi^\top_{l_1}+ \phi_{l_k}\phi^\top_{l_k} ) &= \underbrace{\sum_{k=2}^{2L}(-\eta)^k\sum_{l=1}^L\left(\binom{L+l-2}{k-1} +\binom{L-l}{k-1}+\binom{2l-2}{k-2}\right)}_{\text{sum over combinatorially many terms}}\phi_{l}\phi^\top_{l}.
\end{align}     

\begin{proof}
In the main paper, we outline the steps for transforming the problem of estimating the large summation (i.e., $\sum_{1\le l_1,\ldots,l_k\le L}$) involving many high-order terms (i.e., $\phi_{l_1}\phi^\top_{l_1}\ldots\phi_{l_k}\phi^\top_{l_k}$) into a combinatorial counting problem. Here, we present the details for the remaining cases beyond what was sketched in Figure~\ref{fig:combination-short}, where we focused on one instance of selecting $l_1$ and $l_k$.

More specifically, the element $\phi_{l}\phi^\top_{l}$ arises from the following cases:
\begin{enumerate}[label=(\alph*)]
    \item When $l_1$ selects the $l$-th slot and $l_k$ does \textit{not}, this includes two subcases:
    \begin{enumerate}[label=(\arabic*)]
        \item $l_1$ picks the left $l$-th slot, and the remaining $k-1$ elements (\textit{i.e.,} $l_2,\ldots,l_k$) are selected from a subarray of size $L + l - 2$. See Figure~\ref{fig:combination-short} for a visual example.
        \item $l_1$ picks the right $l$-th slot, and the remaining $k-1$ elements are selected from a subarray of size $L - l$. See Figure~\ref{fig:combi-case2} for a visual example.
    \end{enumerate}
    
    \item When $l_1$ does not pick the $l$-th slot but $l_k$ does, this is symmetric to the previous case and also includes two subcases:
    \begin{enumerate}[label=(\arabic*)]
        \item $l_k$ picks the right $l$-th slot, while the remaining $k-1$ elements (\textit{i.e.,} $l_1,\ldots,l_{k-1}$) are selected from a subarray of size $L + l - 2$. See Figure~\ref{fig:combi-case3}.
        \item $l_k$ picks the left $l$-th slot, and the remaining $k-1$ elements are selected from a subarray of size $L - l$. See Figure~\ref{fig:combi-case4}.
    \end{enumerate}
    
    \item When $l_1$ picks the left $l$-th slot and $l_k$ picks the right $l$-th slot, the intermediate $k-2$ elements (\textit{i.e.,} $l_2,\ldots,l_{k-1}$) are selected from a subarray of size $2l-2$. To ensure that the number of available slots is greater than the number of elements, we require $2l-2 \ge k-2$, which simplifies to $2l \ge k$. See Figure~\ref{fig:comb-case5} for a visual example.
    
\end{enumerate}

Note that cases (a) and (b) yield the same result, while case (c) counts twice because both the first and last elements are the $l$-th slot. Thus, we derive the final result:
\begin{align}
\sum_{k=2}^{2L}(-\eta)^k\sum_{l_1,\ldots,l_k}\frac{1}{2}( \phi_{l_1}\phi^\top_{l_1}+ \phi_{l_k}\phi^\top_{l_k} ) = \underbrace{\sum_{k=2}^{2L}(-\eta)^k\sum_{l=1}^L\left(\binom{L+l-2}{k-1} +\binom{L-l}{k-1}+\binom{2l-2}{k-2}\right)}_{\text{sum over combinatorially many terms}}\phi_{l}\phi^\top_{l}.
\end{align}
\end{proof}
\end{lemma-no}

The following Lemma~\ref{lem:combi-weighted} further simplify the whole summations with many different $n$ choose $k$ in Lemma~\ref{lem:combi-cases}. The obtained result will help us to mitigate the constraint $\eta L < 1/3$ in the final theorem (in Theorem~\ref{thm:main}). 

\subsection{Proof of Lemma~\ref{lem:combi-weighted}} \label{sec:combi-weighted-proof}
\begin{lemma-no}   
For $ \eta \in (0,1)$ and $L >1$, we have:
\begin{align}
  \sum_{k=2}^{2L}(-\eta)^k\left(\binom{L+l-2}{k-1} +\binom{L-l}{k-1}+\binom{2l-2}{k-2}\right)=&(1 - \eta)^{L+l-2} +  (1 - \eta)^{L-l}  +  \eta^2 (1 - \eta)^{2l-2} \\
  & +\eta(2L-2) -2.
\end{align}

\begin{proof}
We start by separating the sum into three parts:
\begin{align}
\sum_{k=2}^{2L} \binom{L+l-2}{k-1} (-\eta)^k + \sum_{k=2}^{2L} \binom{L-l}{k-1} (-\eta)^k + \sum_{k=2}^{2L} \binom{2l-2}{k-2} (-\eta)^k.
\end{align}
Then we use the Binomial expansion formula $n$: $(1-x)^n=\sum_{k=0}^n\binom{n}{k} (-x)^k$. This expression will converge when $|x|<1$.
Set $x=\eta$, we would have: 
\begin{align} 
(1-\eta)^{n}=1+\sum_{k=1}^{n}\binom{n}{k} (-\eta)^k=1-\eta n+\sum_{k=2}^{n}\binom{n}{k} (-\eta)^k. \label{eq:base-binomial}
\end{align}
Seting $n = L+l-2$ and $x = \eta $, we have the result for the first sum:
\begin{align}
\sum_{k=2}^{2L-1} \binom{L+l-2}{k} (-\eta)^k =\sum_{k=2}^{L+l-2} \binom{L+l-2}{k} (-\eta)^k = (1 - \eta)^{L+l-2}  +\eta(L+l-2)-1. 
 \label{eq:part1-combi-case1}
\end{align} 
Note that the binomial coefficient \(\binom{n}{k}\) is defined to be zero when \(k > n\). Because you cannot choose more elements than are available. In the above equation, \( \binom{L+l-2}{k-1} \) is zero for \( k-1 > L + l - 2 \), thus we limit the summation to \( k \leq L + l - 2 \).

Similarly, set $n = L-l$ and $x = \eta $, we have the result for the second sum:
\begin{align}
\sum_{k=2}^{2L}\binom{L-l }{k} (-\eta)^k=\sum_{k=2}^{L-l}\binom{L-l }{k} (-\eta)^k=(1-\eta)^{L-l}+\eta (L-l)-1.  \label{eq:part1-combi-case2}
\end{align} 
Similarly, \( \binom{L-l}{k-1} \) is zero for \( k-1 > L - l \), thus the summation is restricted to \( k \leq L - l \).

In terms of the third sum, we adjust the index and apply the binomial theorem,
\begin{align}
\sum_{k=2}^{2L} \binom{2l-2}{k-2} (-\eta)^k=\eta^2\sum_{k=0}^{2L-2} \binom{2l-2}{k} (-\eta)^{k} =\eta^2\sum_{k=0}^{2l-2} \binom{2l-2}{k} (-\eta)^{k} = \eta^2 (1 - \eta)^{2l-2}.
\end{align}
By combining all three simplified sums, we obtain the final result:
\begin{align}
    (1 - \eta)^{L+l-2} +  (1 - \eta)^{L-l}  +  \eta^2 (1 - \eta)^{2l-2} +\eta(2L-2) -2.
\end{align}
The above result holds when $|\eta|<1$. When $\eta >1$, the series could diverge due to the unbounded growth of the terms. Since the learning rate is always positive, the requirement becomes $\eta\in (0,1)$.
This completes the proof.
\end{proof}
\end{lemma-no}

We aim to establish upper and lower bounds for the expression in Lemma~\ref{lem:combi-weighted} that are independent of \( l \), the results are provided in the following Remark.

\begin{remark} \label{rem:upper-lower-bound}
 For \( 0 < l < L \) and \( \eta \in (0,1) \), the following inequalities hold:
\begin{align}
 (1 - \eta)^{L+l-2} +  (1 - \eta)^{L-l}  +  \eta^2 (1 - \eta)^{2l-2} &> (1 - \eta)^{2L-3} + (1 - \eta)^{L-1} + \eta^2 (1 - \eta)^{2L-4}, \\
 (1 - \eta)^{L+l-2} +  (1 - \eta)^{L-l}  +  \eta^2 (1 - \eta)^{2l-2} &< (1 - \eta)^{L-1} + \eta^2 + 1.
\end{align} 
Within the range \( \eta \in (0,1) \), both the upper and lower bounds are positive.

\begin{proof}
The term \( (1 - \eta)^{L+l-2} \) decreases as \( l \) increases since \( (1 - \eta) < 1 \) for \( \eta > 0 \). Therefore, the maximum occurs at \( l = 1 \), and the minimum at \( l = L-1 \).

Similarly, for \( (1 - \eta)^{L-l} \), the maximum occurs at \( l = L-1 \), and the minimum at \( l = 1 \). For the term \( \eta^2 (1 - \eta)^{2l-2} \), the maximum occurs when \( l = 1 \), and the minimum when \( l = L-1 \).

Thus, the upper bound of the entire expression is achieved by combining the maximum values of each term, resulting in \( (1 - \eta)^{L-1} + \eta^2 + 1 \). The lower bound is given by \( (1 - \eta)^{2L-3} + (1 - \eta)^{L-1} + \eta^2 (1 - \eta)^{2L-4} \).
\end{proof}
\end{remark}

\section{Theoretical Justification of Theorem~\ref{thm:main}} \label{apx:combi-count-proof}

\begin{theorem-no}
Let $\mu$ be the stationary distribution of the state-action pair in the MDP. The following matrix's positive semi-definite inequalities hold: for $\eta\in (0,1)$,
\begin{align} 
\mathbb{E}_{(s,a)\sim\mu} \left[\Gamma_{L}^\top \Gamma_{L}\right]\preceq \left(1-\frac{(\eta (4-2L)  -(1 - \eta)^{L-1} -\eta^2+1)L}{\kappa}\right)\identity.
\end{align}
where the matrix $\Gamma_L$ is defined in Definition~\ref{def:gamma}. Here ``$\preceq$'' is defined between two matrices on both sides (please see Definition~\ref{def:psd}) for the positive semi-definite property.
\end{theorem-no}

\begin{proof} Based on Equation~\eqref{eq:gamma-gamma}, we have:
\begin{align}
\mathbb{E}_{(s,a)\sim\mu}  \left[\Gamma_{L}^\top \Gamma_{L}\right]=\identity-2\eta\mathbb{E}_{(s,a)\sim\mu} \left[\sum_{l=1}^{L}\phi_{l}\phi^\top_{l}\right]+\mathbb{E}_{(s,a)\sim\mu} \left[\sum_{k=2}^{2L}(-\eta)^k\sum_{l_1,\ldots,l_k} \phi_{l_1}\phi^\top_{l_1}\ldots\phi_{l_k}\phi^\top_{l_k} \right].
\end{align}

 By linearity of expectation, the second part of Equation~\ref{eq:gamma-gamma} can be reformulated as 
 \begin{equation}
-2\eta\mathbb{E}_{(s,a)\sim\mu} \left[\sum_{l=1}^{L}\phi_{l}\phi^\top_{l}\right]=-2\eta\sum_{l=1}^{L}\mathbb{E}_{(s,a)\sim\mu} \left(\phi_{l}\phi^\top_{l}\right)=-2\eta L \mathbb{E}_{(s,a)\sim\mu} \left(\phi\phi^\top\right)\preceq \frac{-2\eta L}{\kappa} \identity.
 \end{equation}
The last step is obtained by Remark~\ref{rem:linear-expect}.

Based on the result in the proposed Lemma~\ref{lem:combi-cases}, Lemma~\ref{lem:combi-weighted} and Remark~\ref{rem:upper-lower-bound}, the second part of Equation~\ref{eq:gamma-gamma} can be upper bounded as:
\begin{align}
&\mathbb{E}_{(s,a)\sim\mu} \left[\sum_{k=2}^{2L}(-\eta)^k\sum_{l_1,\ldots,l_k}\frac{1}{2}( \phi_{l_1}\phi^\top_{l_1}+ \phi_{l_k}\phi^\top_{l_k}  ) \right]  \\
&=  \mathbb{E}_{(s,a)\sim\mu} \left[\sum_{k=2}^{2L}(-\eta)^k\sum_{l=1}^L\left(\binom{L+l-2}{k-1} +\binom{L-l}{k-1}+\binom{2l-2}{k-2}\right) \phi_{l}\phi^\top_{l}\right] &\text{by Lemma~\ref{lem:combi-cases}}\\
&=\mathbb{E}_{(s,a)\sim\mu} \left[ \sum_{l=1}^L\sum_{k=2}^{2L}\left(\binom{L+l-2}{k-1} +\binom{L-l}{k-1}+\binom{2l-2}{k-2}\right)(-\eta)^k\phi_{l}\phi^\top_{l}\right] &\text{swap two sums}\\
&=\mathbb{E}_{(s,a)\sim\mu} \left[ \sum_{l=1}^L\left((1 - \eta)^{L+l-2} +  (1 - \eta)^{L-l}  +  \eta^2 (1 - \eta)^{2l-2} +\eta(2L-2) -2\right)\phi_{l}\phi^\top_{l}\right] &\text{By Lemma~\ref{lem:combi-weighted}}\\
&\preceq \left(  (1 - \eta)^{L-1} +\eta^2 +\eta(2L-2)-1\right)\mathbb{E}_{(s,a)\sim\mu} \left[\sum_{l=1}^L\phi_{l}\phi^\top_{l}\right] &\text{By Remark~\ref{rem:upper-lower-bound}}\\
&\preceq\frac{(1 - \eta)^{L-1}L +\eta^2L +\eta(2L-2)L -L}{\kappa}\identity.
\end{align}
Combining the results in the above two inequalities, we finally have the upper bound:
\begin{align}
\mathbb{E}_{(s,a)\sim\mu} \left[\Gamma_{L}^\top \Gamma_{L}\right]\preceq \left(1-\frac{ (\eta (4-2L)  -(1 - \eta)^{L-1} -\eta^2+1)L}{\kappa}\right)\identity.
\end{align}
\end{proof}

Theorem~\ref{thm:main} together its theoretical justification is novel and never used in any analysis relevant to experience replay by the knowledge of the authors.

\section{Sample Complexity Proof of Theorem~\ref{thm:main-cvg}}
We acknowledge that the main structure of convergence proof follows the original work. Here, we made contribution to present a cleaner proof pipeline of the proof and also integrate our tighter bound in Theorem~\ref{thm:main}.
\subsection{Proof of Bias-Variance Decomposition for the Error in Lemma~\ref{lemma:error-decomp}} \label{sec:error-decompose}

\begin{lemma-no}
Let the error terms for every parameter $w$ as the difference between empirical estimation and true MDP: $\varepsilon_i(w)\coloneqq {Q}(s_i,a_i)-Q^*(s_i,a_i)$. For the current iteration $t$, the  difference between current estimated parameter $w$ and the optimal parameter $w^*$ accumulated along the $L$ length sub-trajectory with reverse update is:
\begin{align}
w_{L}-w^*=\underbrace{\Gamma_{L}\left(w_{1}-w^*\right)}_{\text{Bias term}}+\underbrace{\eta\sum_{l=1}^L \varepsilon_l\Gamma_{l-1}\phi_l}_{\text{variance term}}.
\end{align}
\end{lemma-no}
\begin{proof} As shown in Algorithm~\ref{algo:RER} in Lines 5-7,  we use a sampled sub-trajectory of length $L$ to execute $Q$-update reversely at iteration $t$: for $l=1,2,\ldots,L$,
\begin{align}
    w_{l+1}&=w_{l}+\eta \left(r_{L+1-l}+\gamma\max_{a'\in\mathcal{A}}\langle w_{1}, \phi(s_{L+2-l},a')\rangle-\langle w_{l},\phi_{L+1-l}\rangle\right)\phi_{L+1-l}\\
    &=\left(\identity-\eta \phi_{L+1-l}\phi^\top_{L+1-l}\right)w_{l}+\eta \left(r_{L+1-l}+\gamma\sup_{a'\in\mathcal{A}}\langle w_{1}, \phi(s_{L+2-l},a')\rangle\right)\phi_{L+1-l}
\end{align}
where $\identity$ denotes the identity matrix and $w_{1}$ is the parameter of the target network. The second equality is attained by rearranging the terms of the first equality. The $\max$ operator is changed to $\sup$ operator for rigorous analysis purposes. $\langle \cdot, \cdot\rangle$ means inner dot product between two vectors of the same dimension.

Let $Q^*(s_i,a_i)$ be the optimal $Q$-value at state $s_i$ taking action $a_i$ and assume $w^*$ is the optimal parameter, the Bellman optimality equation is written as:
\begin{equation*}
Q^*(s_i,a_i)= R_i+\gamma\mathbb{E}_{s'\sim P(\cdot|s_i,a_i)}\sup_{a'\in\mathcal{A}}\langle w^*, \phi(s',a')\rangle
\end{equation*}
Define the error term $\varepsilon_i(w)$ for parameter $w$ and $i$-th tuple $(s_i,a_i,r_i)$  as the difference between empirical estimation and true probabilistic MDP:
\begin{align}
    \varepsilon_i(w)&\coloneqq {Q}(s_i,a_i)-Q^*(s_i,a_i)\\ &=\left({r}_i+\gamma{\sup_{a'\in\mathcal{A}}\langle w, \phi(s_{i+1}},a')\rangle\right)-\left(R_i+\gamma\mathbb{E}_{s'\sim P(\cdot|s_i,a_i)}\sup_{a'\in\mathcal{A}}\langle w, \phi(s',a')\rangle\right) \\
    &=\left({r}_i-R_i\right)-\gamma\left({\sup_{a'\in\mathcal{A}}\langle w, \phi(s_{i+1}},a')\rangle- \mathbb{E}_{s'\sim P(\cdot|s_i,a_i)}\sup_{a'\in\mathcal{A}}\langle w, \phi(s',a')\rangle\right) \label{eq:error}
\end{align}   
For all $ l=1\ldots,L$, apply the Bellman optimality equation over the RER algorithm over the optimal parameter $w^*$:
\begin{align}
\langle w^*,\phi_{L+1-l}\rangle=R_{L+1-l}+\gamma\mathbb{E}_{s'\sim P(\cdot|s_{L+1-l},a_{L+1-l})}\sup_{a'\in\mathcal{A}}\langle w_{1}, \phi(s',a')\rangle
\end{align}
Right-multiply the above equation on both sides with term $\eta \phi_{L+1-l}$ and combine with the first equation in this proof, we shall get
\begin{align}
w_{l+1}-w^*=\left(\identity-\eta \phi_{L+1-l}\phi^\top_{L+1-l}\right)\left(w_{l}-w^*\right)+\eta\varepsilon_{L+1-l}\phi_{L+1-l}
\end{align}
where the error term $\varepsilon_{L+1-l}$ is defined in Equation~\eqref{eq:error}. Combined with Definition~\ref{def:gamma} and recursively expand the RHS, we shall get the difference \textit{w.r.t.} the optimal one  after reversely updating $L$ consecutive steps:
\begin{align}
w_{L+1}-w^*&=\Gamma_{L}\left(w_{1}-w^*\right)+\eta\sum_{l=1}^L\varepsilon_{l}\Gamma_{l-1}\phi_{l}
\end{align}
The proof is thus finished.
\end{proof}

% [TODO: we need to add the equation mark to the proof.]

\begin{remark} \label{rem:bias-var}
Suppose we execute Algorithm~\ref{algo:RER} for $N$ iterations, we would get:
\begin{align}\label{eq:full-error}
w_{N,L}-w^*&=\underbrace{\prod_{t=N}^1\Gamma_{L}\left(w_{1}-w^*\right)}_{\text{Bias term}}+\underbrace{\eta\sum_{i=1}^{N}\prod_{t=N}^{i+1}\Gamma^l_{L}H^j}_{\text{Variance term}}
\end{align}
where $H^j:=\eta\sum_{i=1}^L \varepsilon_i\Gamma^j_{i-1}\phi_i^j$.
The first term on RHS is noted as the bias, which decays geometrically with $N$ and the second term on RHS is variance along the sub-trajectory of length $L$, which we will later show it has zero mean.
\end{remark}
Note that the above Remark~\ref{rem:bias-var} is relevant to the \citep[Lemma 5]{DBLP:conf/iclr/Agarwal2022}.

\subsection{Bound the Bias Term in Remark~\ref{rem:bias-var}} \label{apx:bias-term-bound}
 We briefly mention here to quantify the upper bound of ${\prod_{t=1}^{N}\Gamma_{L}}$ in Lemma~\ref{lemma:error-decomp}. In comparison to~\citep[lemma 8]{DBLP:conf/iclr/Agarwal2022}, the following Lemma~\ref{lem:contract2}  require $\eta\in(0,1)$ instead of $\eta L<1/3$, which is a relaxation of the original work.
\begin{lemma-no}
Let $\mathbf{x}\in \mathbb{R}^d$ be a non-zero vector and $N$ is the frequency for the target network to be updated. For $\eta\in(0,1), L\in \mathbb{N}$ and $L >1$,   the following matrix's positive semi-definite inequality holds:
\begin{align}
    \mathbb{E}\norm{\prod^{1}_{j=N}\Gamma_{L}\mathbf{x}}^2&\le \exp\left(-\frac{N(\eta (4-2L)L +L  -\eta^2L)}{\kappa}\right) \norm{\mathbf{x}}^2.
\end{align}
Furthermore, the following inequality holds with probability at least $1-\delta$:
\begin{align}
    \mathbb{E}\norm{\prod^{1}_{j=N}\Gamma_{L}\mathbf{x}}_{\phi}^2&\le \exp\left(-\frac{N(\eta (4-2L)L +L  -\eta^2L)}{\kappa}\right)\sqrt{\tfrac{\kappa}{\delta}}\norm{\mathbf{x}}_{\phi}.
\end{align}
The $\phi$-based norm is defined in Definition~\ref{def:phi}.
\end{lemma-no}

\begin{proof} 
In Theorem~\ref{thm:main}, we notice the term $(1-\eta)^{L-1}$ will converge to zero exponentially and thus are omitted under sufficient precision. Thus, we obtain a simplified form
\begin{align}
\mathbb{E}_{(s,a)\sim\mu} \left(\Gamma_{L}^\top \Gamma_{L}\right)&\preceq \left(1-\frac{ \eta (4-2L)L +L -(1 - \eta)^{L-1}L -\eta^2L}{\kappa}\right)\identity\\
 & \preceq\left(1-\frac{\eta (4-2L)L +L  -\eta^2L}{\kappa}\right)\identity. \label{eq:expect_contraction} 
\end{align}
Now, we observe that: 
\begin{equation}\label{eq:sq_norm}
\norm{\prod_{j=N}^{1}\Gamma_{L}\mathbf{x}}^2 = \mathbf{x}^{\top}\prod_{j=1}^{N-1}(\Gamma^j_{L})^{\top}\underbrace{\left(\left(\Gamma^{N}_{L}\right)^{\top}\Gamma^{N}_{L}\right)}_{\text{is independent of the rest}}\prod_{j=N-1}^{1}\Gamma^j_{L} \mathbf{x}
\end{equation}
% $\left(\Gamma^{N}_{L}\right)^{\top}\Gamma^{N}_{L}$ is independent of $\prod_{j=1}^{N-1}(\Gamma_{L})^{\top}$ for $j \leq N-1$. 
Therefore, we take the expectation conditioned on $\Gamma^{j}_{L}$ for $j \leq N-1$ in Equation~\eqref{eq:sq_norm}:
\begin{align}
    \mathbb{E}\norm{\prod_{j=N}^{1}\Gamma^j_{L}\mathbf{x}}^2 \leq \left(1-\frac{ \eta (4-2L)L +L  -\eta^2L}{\kappa}\right)\mathbb{E}\norm{\prod_{j=N-1}^{1}\Gamma^j_{L}\mathbf{x}}^2 & \quad \text{ using Equation~\eqref{eq:expect_contraction}}
\end{align}
Applying the equation above inductively, we conclude the result:
\begin{align}
    \mathbb{E}\norm{\prod_{j=N}^{1}\Gamma^j_{L}\mathbf{x}}^2 &\leq \left(1-\frac{ \eta (4-2L)L +L  -\eta^2L}{\kappa}\right)^N\norm{\mathbf{x}}^2\\
    &\approx \exp\left(-\frac{N(\eta (4-2L)L +L  -\eta^2L)}{\kappa}\right) \norm{\mathbf{x}}^2.
\end{align}
The last step of numerical approximation is obtained by $\lim_{N\to +\infty}(1+\frac{x}{N})^N= \exp(x)$.

We then extend to $\phi$-based norm (in Definition~\ref{def:phi}) as follow: 
\begin{align}
 \norm{\prod_{j=N}^{1}\Gamma_{L}\mathbf{x}}_{\phi} \leq \sqrt{\kappa} \norm{\prod_{j=N}^{1}\Gamma^{j}_{L}\mathbf{x}} &\leq {\exp\left(-\frac{N(\eta (4-2L)L +L  -\eta^2L)}{\kappa}\right)} \sqrt{\kappa} \norm{\mathbf{x}}
\end{align}
Thus, by Markov's inequality, with probability at least $ 1-\delta$, the following event holds:
\begin{equation}
 \norm{\prod_{j=N}^{1}\Gamma_{L}\mathbf{x}}_{\phi} \leq \exp\left(-\frac{N(\eta (4-2L)L +L  -\eta^2L)}{\kappa}\right)\sqrt{\tfrac{\kappa}{\delta}}\norm{\mathbf{x}}_{\phi}
\end{equation}
\end{proof}

\subsection{Bound the Variance Term in Remark~\ref{rem:bias-var}}  \label{apx:variance-bound}
Even though the term $\Gamma_l$ is involved in the expression, it turns out we do not need to modify the original proof and thus we follow the result in the original proof. Please see~\citep[Appendix L.3]{DBLP:conf/iclr/Agarwal2022} for the original proof steps.
\begin{theorem-no}[\cite{DBLP:conf/iclr/Agarwal2022} Theorem 4]\label{lem:linear_variance}
Suppose $\mathbf{x},\mathbf{w}\in \mathbb{R}^d$ are fixed. There exists a universal constant $C$ such that the following event holds with probability at least $1-\delta$: 
 \begin{equation}\label{eq:lin_var_conc}
 \biggr|\langle \mathbf{x},\sum_{t=1}^{N}\prod_{j=N}^{t+1}\Gamma^{j}_{L} H^{j}(\mathbf{w})\rangle\biggr| \leq C\norm{\mathbf{x}}\left(1+\norm{\mathbf{w}}_{\phi}\right)\sqrt{\eta \log\left(\frac{2}{\delta}\right)}
 \end{equation}
 \end{theorem-no}
By a direct application of \citep[Theorem 8.1.6]{vershynin2018high}, we derive the following corollary. Notice that $C$ is the covering number defined in Definition~\ref{def:cover_num}.
 
\begin{corollary}
For fixed parameter $\mathbf{x},\mathbf{w}\in \mathbb{R}^d$, there exists a constant $C$ such that the following inequality holds with probability at least $1-\delta$:
\begin{equation}
\norm{\eta\sum_{i=1}^{N}\prod_{t=N}^{i+1}\sum_{l=1}^L\varepsilon_{l}(\mathbf{w})\Gamma_{l-1}\phi_{l}}_\phi \le C(1+\norm{\mathbf{w}}_\phi)\left(C_{\Phi}+\sqrt{\log\left(\frac{2}{\delta}\right)}\right)   
\end{equation}
The notation $\|\cdot\|_{\phi}$ is mentioned in Definition~\ref{def:phi}.
Here $C_\Phi$ is the covering number with $L_2$ norm in $\mathbb{R}^d$, see Theorem~\ref{theorem:dudley} for details.
 \end{corollary}

\subsection{Overall Sample Complexity Analysis}
\begin{lemma} \label{lem:final}
There exists constants $C_\Phi,C_1,C_2$, such that: 
\begin{enumerate}
    \item $C_1\kappa\log \left(\frac{T\kappa}{N\delta(1-\gamma)}\right)< N$;
    \item $\eta\le C_2\frac{(1-\gamma)^2}{C^2_{\Phi}+\log \left(\frac{T}{N\delta}\right)}$.
\end{enumerate}
Let $1\le k\le \frac{T}{N}$, where $k$ is an index for the target network and $T/N$ is the total number of target network updates. The following holds with probability at least $1-\delta$: 
\begin{enumerate}
    \item For the target network, $\|\theta_k\|\le\frac{4}{1-\gamma}$;
    \item For the error accumulated along $L$-consecutive steps of reverse update,
    \begin{equation*}
    \begin{aligned}
    \|w_{k,L}-w^*\|_{\phi}\le &\sqrt{\frac{25T\kappa}{N\delta(1-\gamma)^2}}\exp\left(-\frac{ N((\eta (4-2L)L +L  -\eta^2L))}{\kappa}\right) \\
    &+\sqrt{\frac{\eta \left(C^2_\Phi+\log\left(\frac{T}{N\delta}\right)\right)}{(1-\gamma)^2}}
    \end{aligned}
\end{equation*}

\end{enumerate}
When we combine the above two cases, we have:
\begin{equation*}
\begin{aligned}
\|Q^T-Q^*\|_{\infty}\le&\mathcal{O}\left(\frac{\gamma^{T/N}}{1-\gamma}\right)\\
    &+\mathcal{O}\left(\sqrt{\frac{T\kappa}{N\delta(1-\gamma)^4}}\exp\left(-\frac{N(\eta (4-2L)L +L  -\eta^2L)}{\kappa}\right)\right)\\
    &+\mathcal{O}\left(\sqrt{\frac{\eta \left(C^2_\Phi+\log\left(\frac{T}{N\delta}\right)\right)}{(1-\gamma)^4}}\right)   
\end{aligned}
\end{equation*}
We can obtain the sample complexity of the whole learning framework by setting the RHS as $\varepsilon$ and we shall recover the result we show in Theorem~\ref{thm:main-cvg}.
\end{lemma}

\section{Extra Notations and Definitions} \label{apx:extra}

\begin{definition}[$\varepsilon$-Net, Covering Number and Metric Entropy] \label{def:cover_num}
Given metric space $(\mathbb{R}^d,\|\cdot\|_2)$, consider a region $\mathcal{K}\subset \mathbb{R}^d$. 1) A subset of Euclidean balls $\mathcal{N}$ is called an $\varepsilon$-Net of $\mathcal{K}$  (for $\mathcal{N}\subseteq\mathcal{K}, \varepsilon>0$) if every point $x\in\mathcal{K}$ is within $\varepsilon$ distance of a point of $\mathcal{N}$:
\begin{align*}
    \exists x'\in \mathcal{N}, \|x-x'\|_2\le \varepsilon, \text{ for all } x\in \mathcal{K}
\end{align*}
2) Equivalently, we denote $\mathcal{N}(\mathcal{K},\varepsilon)$ as the smallest number of closed balls with centers in $\mathcal{K}$ and radius $\varepsilon$ whose union covers $\mathcal{K}$. 3) Metric Entropy. It is the Logarithm of the covering number $\log_2\mathcal{N}(\mathcal{K},\varepsilon)$.
\end{definition}

\begin{theorem}[Dudley's Integral Inequality] \label{theorem:dudley}
Let $\{\mathbf{x}_t\in\mathbb{R}^d\}_{t=1}^N\in\mathcal{K}$ be a random process with zero means on the metric space $(\mathbb{R}^d,\|\cdot\|_2)$ with sub-gaussian increments. Then
\begin{align*}
    \mathbb{E}\sup_{\mathbf{x}_t\in\mathcal{K}} x_t\ge CK\int_{0}^\infty\sqrt{\log_2\mathcal{N}(\mathcal{K},\varepsilon)}\mathit{d}\varepsilon
\end{align*}
\end{theorem}

\begin{lemma}[Cauchy-Schwarz Inequality]  \label{lem:cauchy} Based on Assumption~\ref{asump:phi-kappa} that the feature vector for state-action is bounded $\phi(s,a)^\top \phi(s,a)\le 1$, for all $(s,a)\in\mathcal{S}\times \mathcal{A}$.
For $1\le l,l'\le L$, we have:
\begin{align}
|\phi_{l}^\top \phi_{l'}|^2\le \phi_{l}^\top \phi_{l} \cdot   \phi_{l'}^\top \phi_{l'} \le 1.
\end{align}
\end{lemma}

\begin{definition}[Positive Semi-Definite Property of Matrix] \label{def:psd}
For two symmetric matrices $A, B \in \mathbb{R}^{d\times d}$, we say $A \preceq B$ if $B - A$ is positive semi-definite: $\mathbf{x}^\top (B-A)\mathbf{x}\ge 0$, for all $\mathbf{x}\in\mathbb{R}^d$ non-zero $d$-dimensional vector.
\end{definition}

\begin{lemma}[AM–GM Inequality] \label{lem:am-gm}
The inequality of arithmetic and geometric means, or more briefly the AM–GM inequality, states that the arithmetic mean of a list of non-negative real numbers is greater than or equal to the geometric mean of the same list.
For $x,y\in \mathbb{R}$, $|xy|\le \frac{1}{2}(x^2+y^2)$.
\end{lemma}

\begin{lemma}[Recursive Formula] \label{eq:yanghui}
Let $m,n\in \mathbb{N}$ be positive integers.  $\binom{n}{k}$ denotes a binomial coefficient, which is computed as $\frac{n!}{k!(n-k)!}$. Then for all $1\le k\le n-1$, we have:
\begin{equation*}
 \binom{n -1}{k} + \binom{n -1}{k-1}= \binom{n  }{k}
\end{equation*}
\end{lemma}

\begin{lemma}[Rising Sum of Binomial Coefficients]
Let $m,n\in \mathbb{N}$ be positive integers. Then:
\begin{equation*}
\sum_{j = 0}^m  \binom{n + j }{n} =  \binom{n + m + 1}{n + 1} = \binom{n + m + 1}{m}
\end{equation*}
\end{lemma}

\begin{lemma}[Vandermonde identity \citet{haddad2021repeated}]  \label{lem:repeated-comb}
For any $k, q, n \in \mathbb{N}$ such that $n \ge q$, we have:
\begin{equation*}
\sum_{i=q}^n\binom{i }{k}=\binom{n+1}{k+1} - \binom{q }{k+1}
\end{equation*}
\end{lemma}

\begin{lemma}  \label{lem:repeated-comb2}
For any $k, q, n \in \mathbb{N}$ such that $n \ge q$, we have:
\begin{equation*}
A_n=\sum_{i=q}^n\binom{2i }{k}\qquad B_n=\sum_{i=q}^n\binom{2i-1 }{k}
\end{equation*}
$A_n+B_n$ has a closed form, $A_n-B_n$ also has a closed form.
\end{lemma}

\begin{lemma}[Linearity of Expectation] For random variables $X_1,X_2,\ldots,X_n $
  and constants $c_1,c_2,\ldots,c_n$, we have:
  \begin{equation*}
    \mathbb{E}\sum_{i=1}^nc_ix_i=\sum_{i=1}^nc_i\mathbb{E}X_i
  \end{equation*}
\end{lemma}

\end{document}